\documentclass{article}
\usepackage[preprint,nonatbib]{neurips_2020}
\usepackage{amsmath}
\usepackage{subcaption}
\usepackage{mathtools}
\usepackage{graphicx}
\usepackage[title]{appendix}
\usepackage{parskip}
\usepackage{hyperref}

\usepackage{url} 
\usepackage{booktabs}       
\usepackage{nicefrac}       
\usepackage{subcaption}
\usepackage{amsthm}
\usepackage{amssymb}
\usepackage{amsfonts}
\usepackage{fontawesome}

\newtheorem{problem*}{Problem}
\newtheorem{theorem}{Theorem}
\newtheorem{lemma}{Lemma}

\newtheorem{definition}{Definition}

\newtheorem{example}{Example}
\newtheorem*{example*}{Example}
\usepackage[ruled,linesnumbered,noresetcount,vlined]{algorithm2e}
\makeatletter

\newcommand{\nosemic}{\renewcommand{\@endalgocfline}{\relax}}
\newcommand{\dosemic}{\renewcommand{\@endalgocfline}{\algocf@endline}}
\let\oldnl\nl
\newcommand{\nonl}{\renewcommand{\nl}{\let\nl\oldnl}}

\SetKwFor{Repeat}{repeat}{:}{endw}
\makeatother
\usepackage{xcolor}
\usepackage{setspace}
\usepackage{multirow}
\usepackage{bm,bbm}
\usepackage{paralist}
\usepackage{enumitem}

\DeclareMathOperator*{\argmax}{argmax}
\DeclareMathOperator*{\argmin}{argmin}

\newcommand{\cA}{\mathcal{A}}  
 \newcommand{\cD}{\mathcal{D}}

\newcommand{\cI}{\mathcal{I}} \newcommand{\cL}{\mathcal{L}}
\newcommand{\cM}{\mathcal{M}} \newcommand{\cN}{\mathcal{N}}
 
\newcommand{\cR}{\mathcal{R}}

 \newcommand{\cY}{\mathcal{Y}}

\newcommand{\EE}{\mathbb{E}} \newcommand{\RR}{\mathbb{R}}

\makeatletter
\newcommand{\oset}[3][0ex]{%
  \mathrel{\mathop{#3}\limits^{
    \vbox to#1{\kern-2\ex@
    \hbox{$\scriptstyle#2$}\vss}}}}
\makeatother
\newcommand{\optimal}[1]{\oset{\scalebox{.5}{$\star$}}{#1}}

\title{Differentially Private and Fair Deep Learning: \\
A Lagrangian Dual Approach}

\author{%
  Cuong Tran \\
  Syracuse University \\
  \texttt{cutran@syr.edu}
  \And 
  Ferdinando Fioretto \\
  Syracuse University \\
  \texttt{ffiorett@syr.edu} 
  \And
  Pascal Van Hentenryck\\
  Georgia Institute of Technology\\
  pvh@isye.gatech.edu
}

\begin{document}
\maketitle\sloppy\allowdisplaybreaks

\begin{abstract}
A critical concern in data-driven decision making is to build models whose outcomes do not discriminate against some demographic groups, including gender, ethnicity, or age. To ensure non-discrimination in learning tasks, knowledge of the sensitive attributes is essential, while, in practice, these attributes may not be available due to legal and ethical requirements. 
To address this challenge, this paper studies a model that protects the privacy of the individuals’ sensitive information while also allowing it to learn non-discriminatory predictors. 
The method relies on the notion of differential privacy and the use of Lagrangian duality to design neural networks that can accommodate fairness constraints while guaranteeing the privacy of sensitive attributes. 
The paper analyses the tension between accuracy, privacy, and fairness and 
the experimental evaluation illustrates the benefits of the proposed model on several prediction tasks.
\end{abstract}

\section{Introduction}
\label{sec:introduction}

A number of socio-technical decisions, such as criminal assessment, landing, and hiring, are increasingly being aided by machine learning systems. 
A critical concern is that the learned models are prone to report outcomes that are discriminatory against some demographic group, including gender, ethnicity, or age. 
These concerns have spurred the recent development of fairness definitions and algorithms for decision-making, focusing attention on the tradeoff between the model accuracy and fairness. 

To ensure non-discrimination in learning tasks, knowledge of the \emph{sensitive} attributes is essential. At the same time, legal and ethical requirements often prevent the use of this sensitive data. 
For example, U.S.~law prevents using racial identifiers in the development of models for consumer lending or credit scoring. 
Other requirements may be even more stringent, and prevent the collection of protected user attributes, such as for the case of racial attributes in the E.U.~General Data Protection Regulation (GDPR), or require protection of the consumer data privacy. 
In this scenario, an important tension arise between (1) the demand for models to be non-discriminatory, (2) the requirement for such model to use the protected attribute during training, 
and (3) the restriction on the data or protected attributes that can be used. 
There is thus a need to provide learning models that can both guarantee non discriminatory decisions and protect the privacy of the individuals' sensitive attributes.

To this end, this paper introduces a differential privacy framework to train deep learning models that satisfy several group fairness notions, including \emph{equalized odds}, \emph{accuracy parity}, and \emph{demographic parity} \cite{zafar:17, hardt:16, pmlr-v80-agarwal18a}, while providing privacy of the protected attributes. The key elements 
of the framework can be summarized as follows:
\begin{enumerate}
\item The fairness requirement is captured by casting the learning task as a constrained optimization problem. A Lagrangian dual approach is then applied to the learning task, dualizing the fairness constraints using augmented Lagrangian terms \cite{Hestenes:69}.

\item The privacy requirement is enforced by using a \emph{clipping approach} on the primal and dual steps and adding noise calibrated by the sensitivities of the constraint terms and their gradients. The primal step only applies clipping on constraint gradients involving sensitive attributes, 
and thus, does not have a major effect on the model accuracy.

\item The framework addresses the bias-variance trade-off of clipping by providing bounds on the expected errors of constraint gradients and constraint violations. The clipping bounds can then be calibrated by minimizing these upper bounds.

\item Finally, the framework generalizes to the important scenario where only a subset of the individuals reports the sensitive attributes, i.e., when participants are given the choice of releasing sensitive information.
\end{enumerate}

\noindent
The rest of the paper reviews the related work (Section \ref{sec:related_work}), introduces the problem setting (Section \ref{sec:problem}), discusses the fairness and privacy definitions adopted (Section \ref{sect:preliminaries}), presents the proposed \emph{Private and Fair Lagrangian Dual} (PF-LD) framework (Sections \ref{sec:lagrangian_framework}, \ref{sec:PF-LD}, and \ref{sec:missing_values}), its theoretical results (Sections \ref{sec:privacy_analysis} and \ref{sec:bias-variance}), and its empirical evaluation on several prediction tasks (Section \ref{sec:experiment}). The empirical results show that, on selected benchmarks, PF-LD achieves an excellent trade-off among accuracy, privacy, and fairness. It may represent a promising step towards a practical tool for privacy-preserving and fair decision making.

\section{Related Work}
\label{sec:related_work}
The line of works on algorithmic fairness can be categorized into three groups those adopting pre-processing techniques to guarantee fairness~\cite{zhao2019conditional, edwards2015censoring, beutel2017data}, those developing modifying algorithms to satisfy some fairness notion~\cite{10.5555/3327144.3327203, woodworth2017learning, NIPS2017_6670, DBLP:conf/icdm/CaldersKKAZ13, 10.1145/3306618.3314255}, and those using post-processing techniques~\cite{Feldman2015ComputationalFP, NIPS2016_6374, DBLP:journals/corr/abs-1812-06135, 10.1145/3306618.3314287}. The interested reader is referred to \cite{alex2018frontiers} for a recent survey on algorithmic fairness. 
On the differentially private deep learning front, there are two relevant lines of work. The first is based on the seminal work of Abadi et al.~\cite{abadi:16}, which derives a differential privacy version of stochastic gradient descent (DP-SGD) and proposes a technique, called \emph{moment accountant} to track detailed information of the privacy loss incurred by a sequence of SGD steps \cite{abadi:16}. This idea has been extended and improved by a number of follow-up work \cite{mcmahan:18,pichapati:19}. 
The second avoids using the iterative nature of the optimization algorithms used to training deep learning models by exploiting a collection of \emph{teacher} models \cite{papernot:16,berthelot2019mixmatch} to train a privacy-preserving model.

While this literature is extensive, the topics of privacy and fairness have been study mostly in isolation. A few exceptions are represented by the following work.
The work by Dwork et.~al~\cite{dwork:12} is one of the earliest contribution linking fairness and differential privacy and shows that individual fairness is a generalization of differential privacy. 
More recently, Cummings et.~al~\cite{cummings:19} consider the tradeoffs when considering differential privacy and equal opportunity, a notion of fairness that restricts a classifier to produce equal true positive rates across different groups. The work claim that there is no classifier that achieves $(\epsilon,0)$-differential privacy, satisfies equal opportunity, and has accuracy better than a constant classifier. 
Ekstrand et.~al~\cite{ekstrand:18} raise questions about the tradeoffs involved between privacy and fairness and, finally, Jagielski et.~al~\cite{Jagielski:20} shows two simple, yet effective algorithms that satisfy $(\epsilon, \delta)$-differential privacy and equalized odds. 
Finally, a recent line of work has also observed that private models may have a negative impact towards fairness. 
In particular, Pujol et.~al~\cite{pujol:20} shows that differential privacy could disproportionately affect some groups on several Census resource allocation tasks. A similar observation was made by Bagdasaryan et.~al~\cite{bagdasaryan:19} in the context of private deep learning models trained using DP-SDG. The authors observed disparity in performance across different sub-populations on several classification tasks.

It is also worth mentioning that in order to build a fair model it is necessary to collect a subset of users with their sensitive information like their gender, races or ages. This poses a privacy risks on fair models \cite{Jagielski:20, pmlr-v80-kilbertus18a}. To achieve a fair model without disclosing those sensitive  attributes, <ozannar et.~al~\cite{mozannar2020fair} have very recently developed a differential privacy mechanism in which the true sensitive information is perturbed prior being applied to a fair learning model. 

In contrast to the work discussed above, this paper, presents a Lagrangian dual method to enforce several fairness constraints directly into the training cycle of a deep neural network and proposes a differentially private and fair version of the learning algorithm. 

\section{Problem Settings and Goals}
\label{sec:problem}

The paper adopts boldface symbols to describe vectors (lowercase) and matrices (uppercase). Italic symbols are used to denote scalars (lowercase) and random variables or data features (uppercase). Notation $\|\cdot\|$ is used to denote the $l_2$ norm.

The paper considers datasets $D$ consisting of $n$ individual data points $(X_i, A_i, Y_i)$, with $i \!\in\! [n]$ drawn i.i.d.~from an unknown distribution. Therein, $X_i \!\in\! \mathcal{X}$ is a \emph{non-sensitive} feature vector, $A_i \!\in\! \mathcal{A}$, with $\mathcal{A} = [m]$ (for some finite $m$) is a protected attribute, and $Y_i \!\in\! \mathcal{Y} = \{0,1\}$ is a binary label.
The goal is to learn a classifier $\cM_\theta : \mathcal{X} \to \mathcal{Y}$, where $\theta$ is a vector of real-valued parameters, 
that ensures a specified non-discriminatory notion with respect to $A$ while guaranteeing the \emph{privacy} of the sensitive attribute $A$. 
The model quality is measured in terms of a nonnegative, and assumed differentiable, \emph{loss function} $\mathcal{L}: \mathcal{Y} \times \mathcal{Y} \to \mathbb{R}_+$, and the problem is that of minimizing the empirical risk function:
{
\begin{equation}
\label{eq:erm}
    \min_\theta J(\cM_\theta, D) = \frac{1}{n} \sum_{i=1}^n 
    \mathcal{L}(\cM_\theta(X_i), Y_i).
    \tag{L}
\end{equation}
}
\setcounter{equation}{0}
The paper focuses on learning general classifiers, such as neural networks, that satisfy group fairness (as defined next) and protect the disclosure of the sensitive attributes using the notion of differential privacy.
Importantly, the paper assumes that the attribute $A$ is not part of the model input during inference. This is crucial in the application of interest to this work as the protected attributes cannot be disclosed.

\section{Preliminaries} \label{sect:preliminaries}
This section reviews the fairness and privacy notion adopted in this work.

\subsection{Fairness}    \label{sec:fairness}
The paper consider a classifier $\cM$ satisfying some group fairness notion under a distribution over $(X, A, Y)$ for the protected attribute $A$ and focuses on three fairness notions:

\begin{itemize}[leftmargin=*, parsep=2pt, itemsep=0pt, topsep=0pt]
	\item \emph{Demographic Parity}: 
	$\cM$'s predictions are statistically independent of the protected attribute $A$. That is, 
	{	
	\begin{equation*} 
	    \Pr[\cM(X) \!=\! \hat{y} \mid A \!=\! a] \!=\! \Pr[\cM(X) \!=\! \hat{y}]\;\; 
	    \forall a \in \cA, \hat{y} \in \cY,
	\end{equation*}
	}	
	which, since $\hat{y} \in \{0,1\}$, can be expressed as
	{
	\begin{equation*} 
		\EE[\cM(X) \mid A = a] = \EE[\cM(X)], \;\;
		\forall a \in \cA.
	\end{equation*}
	}
	\item \emph{Equalized odds}: $\cM$'s predictions are conditionally independent of the protected attribute $A$ given the label $Y$. 
	That is, for all $a \in \mathcal{A}, \hat{y} \in \mathcal{Y}$, and $y \in \mathcal{Y}$:
	{
	\begin{equation*}
    \Pr[\cM(X) \!=\! \hat{y} \mid A\!=\!a, Y\!=\!y] \!=\! 
    \Pr[\cM(X) \!=\! \hat{y} \mid Y\!=\!y].
	\end{equation*}
	}
	or, equivalently, for all $a \in \cA, y \in \cY$,
	{
	\begin{equation*}
	\EE[\cM(X) \mid A\!=\!a, Y \!=\! y] \!=\! 
	\EE[\cM(X) \mid Y\!=\!y].
	\end{equation*}
	}
	\item \textit{Accuracy parity}:
	$\cM$'s miss-classification rate is conditionally independent of the protected attribute: 
	{
	\begin{equation*}
		\Pr[\cM(X) \neq Y \mid A=a] = \Pr[\cM(X) \neq Y], \;\;
		\forall a \in \cA,
	\end{equation*}
	}
	or equivalently,
	{
	\begin{equation*}
	\label{eq:ap}
	  	\EE[\mathcal{L}(\cM(X), Y) \mid A \!=\! a] 
	  \!=\! \EE[\mathcal{L}(\cM(X), Y)], \;\;
	  \forall a \in \cA,
	\end{equation*}
	}
	where $\mathcal{L}$ is the loss function to minimize in problem \eqref{eq:erm}.
\end{itemize}

As noted by \cite{agarwal:18} and \cite{fioretto:20b}, 
several fairness notions, including those above, can be viewed as  equality constraints between the properties of each group with respect to the population. These constraints can be expressed as:
\begin{equation}
    \label{eqn:equality_constraint}
    \EE_{z \sim D_{P_i}}[h(z)] 
    -
    \EE_{z \sim D_{G_i}}[h(z)] = 0
\end{equation}

where, for $i$ in some index set $\cI$, $D_{P_i}$ is a subset of the dataset $D$, indicating the \emph{population term}, $D_{G_i}$ is a subset of $D_{P_i}$, indicating the \emph{group term}, and is obtained by accessing the protected attributes $A$, the function $h$ characterizes the model output under some fairness definition.

\begin{example}[Demographic parity]
	Demographic parity can be expressed as a set of $|\cA|$ constraints, with 
	$h(z) \!=\! \cM_\theta(z)$ and, for each $i \!\in\! \cA$, 
	the subsets indicating population terms are defined as:
	$$D_{P_i} \!=\! \{(X,Y) \mid (X,A,Y) \!\in\! D\},$$ 
	and the subsets indicating the group terms as: 
	$$D_{G_i} \!=\! \{(X,Y) \mid (X,A,Y) \!\in\! D \land A \!=\! i\}.$$
\end{example}

\begin{example}[Equalized odds]
	Equalized odds can be expressed as a set of $2|\cA|$ constraints, with $h(z) \!=\! \cM_\theta(z)$, and for each choice of $y \in \{0,1\}$, and $i \!\in\! \cA$, 
	the subsets indicating population terms are defined as:
	$$
	D_{P_i} \!=\! \{(X,Y) \mid (X,A,Y) \!\in\! D \land Y \!=\! y \},
	$$ 
	and the subsets indicating the group terms as: 
 	$$
 	D_{G_i} \!=\! \{(X,Y) \mid (X,A,Y) \!\in\! D \land Y \!=\! y \land A \!=\! i\}.
 	$$
\end{example}

\begin{example}[Accuracy parity]
	Accuracy parity can be expressed as a set of $|\cA|$ constraints, with $h(z) \!=\! \cL_\theta(\cM_\theta(z))$,
	where $\cL$ is the loss function defined in problem \eqref{eq:erm}, and, for each $i \!\in\! \cA$, 
	the subsets indicating population terms are defined as:
	$$
	D_{P_i} \!=\! \{(X,Y) \mid (X,A,Y) \!\in\! D\},$$ 
	and the subsets indicating the group terms as: 
	$$ 
	D_{G_i} \!=\! \{(X,Y) \mid (X,A,Y) \!\in\! D \land A \!=\! i\}.
	$$
\end{example}

\subsection{Differential Privacy}
\label{sec:differential_privacy}
Differential privacy (DP) \cite{dwork:06} is a strong privacy notion used to quantify and bound the privacy loss of an individual participation to a computation. 
While traditional DP protects the participation of an individual to a dataset used in a computation, similarly to \cite{Jagielski:20,mozannar2020fair}, this work focuses on the instance where the protection is restricted to the sensitive attributes only. 
A dataset $D \!\in\! \mathcal{D} \!=\! (\mathcal{X} \!\times\! \mathcal{A} \!\times\! \mathcal{Y})$ of size $n$ can be described as a pair $(D_P, D_S)$ where 
$D_P \!\in\! (\mathcal{X} \!\times\! \mathcal{Y})^n$ describes the \emph{public} attributes and $D_S \!\in\! \mathcal{A}^n$ describes the sensitive attributes. 
\emph{The privacy goal is to guarantee that the output of the learning model does not differ much when a single individual sensitive attribute is changed}.

The action of changing a single attribute from a dataset $D_S$, resulting in a new dataset $D_S'$, defines the notion of \emph{dataset adjacency}. Two dataset $D_S$ and $D_S' \in \mathcal{A}^n$ are said adjacent, denoted $D_S \sim D_S'$, if they differ in at most a single entry (e.g., in one individual's group membership).

\begin{definition}[Differential Privacy]
	\label{dp-def}
	A randomized mechanism $\mathcal{M} \!:\! \mathcal{D} \!\to\! \mathcal{R}$ with domain $\mathcal{D}$ and range $\mathcal{R}$ is $(\epsilon, \delta)$-differentially private w.r.t.~attribute $A$, if, for any dataset  $D_P \!\in\! (\mathcal{X} \times \mathcal{Y})^n$, any two adjacent inputs $D_S, D_S' \!\in\! \mathcal{A}^n$, and any subset of output responses $R \subseteq \mathcal{R}$:
	{
	\[
	    \Pr[\mathcal{M}(D_P, D_S) \in R ] \leq  e^{\epsilon} 
	    \Pr[\mathcal{M}(D_P, D_S') \in R ] + \delta.
	\]
	}
\end{definition}
\noindent When $\delta\!=\!0$ the algorithm is said to satisfy $\epsilon$-differential privacy. 
Parameter $\epsilon > 0$ describes the \emph{privacy loss} of the algorithm, 
with values close to $0$ denoting strong privacy, while parameter 
$\delta \in [0,1]$ captures the probability of failure of the algorithm to 
satisfy $\epsilon$-differential privacy. The global sensitivity $\Delta_f$ of a real-valued 
function $f: \mathcal{D} \to \mathbb{R}^k$ is defined as the maximum amount 
by which $f$ changes  in two adjacent inputs $D$ and $D'$:
\(
	\Delta_f = \max_{D \sim D'} \| f(D) - f(D') \|.
\)
In particular, the Gaussian mechanism, defined by
{
	\[
    \mathcal{M}(D) = f(D) + \mathcal{N}(0, \Delta_f^2 \, \sigma^2), 
\]
}
\noindent where $\mathcal{N}(0, \Delta_f\, \sigma^2)$ is 
the Gaussian distribution with $0$ mean and standard deviation 
$\Delta_f\, \sigma^2$, satisfies $(\epsilon, \delta)$-DP for 
$\delta \!>\! \frac{4}{5} \exp(-(\sigma\epsilon)^2 / 2)$ 
and $\epsilon \!<\! 1$ \cite{dwork:14}.

\section{Constrained Learning with Lagrangian Duality}
\label{sec:lagrangian_framework}
When interpreted as constraints of the form  \eqref{eqn:equality_constraint}, fairness properties can be explicitly 
imposed to problem \eqref{eq:erm}, resulting in a constrained empirical risk minimization problem. 
Solving this new problem, however, becomes challenging due to the presence  of constraints. To address this challenge, this work 
uses concepts borrowed from Lagrangian duality. 

Consider a set of $|\cI|$ constraints of the form \eqref{eqn:equality_constraint}, and expressed succinctly as: 
\begin{equation}
\label{eq:fairness_constraints}
    \bm{\mu}(D_P) - \bm{\mu}(D_G) = \bm{0}^\top,
\end{equation}
where $\bm{\mu}(D_P)$ and $\bm{\mu}(D_G)$ are vectors containing elements  
$\mu(D_{P_i}) \!=\! \hat{\EE}_{z \sim D_{P_i}}[h(z)]$
and
$\mu(D_{G_i}) \!=\! \hat{\EE}_{z \sim D_{G_i}}[h(z)]$, respectively,
for each $i \!\in\! \cI$.
Notice that the constraints in $\bm{\mu}(D_P)$ access public data only, while the constraints in $\bm{\mu}(D_G)$ access also the sensitive data. The resulting learning task is defined by the following optimization problem
{
\begin{subequations}
\label{eq:fair_learn}
	\begin{flalign}
	  \argmin_\theta &\; J (\cM_\theta, D_P) =
	  \frac{1}{n} \sum_{i=1}^n {\cal L}({\cal M}_\theta(X_i), Y_i) 
	  \label{eq:fair_lern_a}\\
	  \mbox{ subject to }
	  & \hspace{30pt} 
	  \bm{\mu}(D_P) - \bm{\mu}(D_G) = \bm{0}^\top.
	  \label{eq:fair_lern_b}
	\end{flalign}
\end{subequations}
}
\noindent In \emph{Lagrangian relaxation}, the problem constraints
are relaxed into the objective function using \emph{Lagrangian
multipliers} $\lambda_i \geq 0$ associated to each of the $|\cI|$ constraints and expressing the penalty induced by violating them. 
When all the constraints are relaxed, the \emph{Lagrangian 
function} becomes
{
\begin{equation}
\label{eq:lagrangian_function}
	\cL_{\bm{\lambda}}(\theta) = J(\cM_\theta, D_P) + 
	\bm{\lambda}^\top \left|
	\bm{\mu}(D_P) - \bm{\mu}(D_G) 
	\right|,
\end{equation}
}
where $\bm{\lambda} = (\lambda_1, \ldots, \lambda_{|\cI|})$ 
and the function $|\cdot|$, used here to denote 
the element-wise operator (i.e, $|\mu(D_{P_i}) - \mu(D_{G_i})|$ for $i \!\in\! \cI$), captures a quantification of the 
constraint violations, often used in constraint programming 
\cite{Fontaine:14}. 

Using a Lagrangian function, the optimization becomes
{
\begin{equation}
\label{eq:LR}
  \optimal{\theta}(\bm{\lambda}) = LR_{\bm{\lambda}} = \argmin_\theta \cL_{\bm{\lambda}}(\theta),
\end{equation}
}
that produces an approximation $\cM_{\optimal{\theta}(\bm{\lambda})}$ of $\cM_{\optimal{\theta}}$. The Lagrangian dual finds the best Lagrangian multipliers, i.e.,
{
\begin{equation}
\label{eq:LD}
	\textstyle \optimal{\bm{\lambda}} = \argmax_{\bm{\lambda} \geq 0} J(\cM_{\optimal{\theta}(\bm{\lambda})}, D_P),
\end{equation}
}
to obtain $\cM_{\optimal{\theta}(\optimal{\bm{\lambda}})}$, 
i.e., the strongest Lagrangian relaxation of $\cM$. 
Learning this relaxation relies on an iterative scheme that
interleaves the learning of a number of Lagrangian relaxations (for
various multipliers) with a subgradient method to learn the best
multipliers. The resulting method, called \emph{Fair-Lagrangian Dual} (F-LD) is sketched in Algorithm \ref{alg:alg1}. 
Given the input dataset $D$, the optimizer step size $\alpha \!>\! 0$, 
and step sizes $\bm{s}$, the Lagrangian
multipliers are initialized in line \ref{line:1a}. The training is
performed for a fixed number of $T$ epochs. At each epoch $k$, the
\emph{primal update} step (lines \ref{line:3a} and \ref{line:4a}) 
optimizes the model parameters $\theta$ using stochastic gradient 
descent over different mini-batches $B \subseteq D$.  The optimization 
step uses the current Lagrangian multipliers $\bm{\lambda}_k$. 
Therein, $B_P$ and $B_G$ indicate the population and group terms over a minibatch. 
After each epoch, the \emph{dual update} step (line \ref{line:5a}), 
updates the value of the Lagrangian multipliers following to a 
\emph{dual ascent} rule \cite{boyd2011distributed,DBLP:conf/aaai/FiorettoMH20}. The multipliers 
values are thus restricted to a predefined upper bound 
$\lambda^{\max}\!$ (line \ref{line:6a}). 
{\small
\begin{algorithm}[t]
  \caption{Fair-Lagrangian Dual (F-LD) \!\!\!\!\!\!\!\!\!\!\!\!\!\!\!\!}
  \label{alg:alg1}
  \setcounter{AlgoLine}{0}
  \SetKwInOut{Input}{input}

  \Input{$D=(X_i, A_i, Y_i)_{i=1}^n:$ Training data; \\
       $\alpha, \bm{s} = (s_1, s_2, \ldots):$ step sizes.\\
       $\lambda^{\max}$: Max multipliers value.}
  \label{line:1a}
  $\lambda_{1,i} \gets 0 \;\; \forall i \in \cI$\\
  \For{epoch $k =  1,2, \ldots T$} {
  \label{line:2a}
    \ForEach{Mini-batch $B \subseteq D$}{
    \label{line:3a}
      $\theta \!\gets\! \theta - 
      \alpha \nabla_{\theta} \big[
      J(\cM_\theta, B_P) \text{+} \bm{\lambda}_k^{\top}\!
      \left|\bm{\mu}(B_P) \,\text{-}\, \bm{\mu}(B_G)\right|
      \big]\!\!\!\!\!\!\!\!\!\!\!\!\!\!\!
      $
      \label{line:4a}
    }
    $\bm{\lambda}_{k+1} \gets \bm{\lambda}_k + s_k\, 
        \left|
        \bm{\mu}(D_P) - \bm{\mu}(D_G)
        \right|$\\
    \label{line:5a}
    $\lambda_{k+1,i} \gets \min(\lambda^{\max}, \lambda_{k+1,i}) \;\; \forall i \in \cI$
    \label{line:6a}
  }
\end{algorithm}
}
\section{A Private and Fair LD Model}
\label{sec:PF-LD}
To ensure fairness, the primal (line \ref{line:4a}) and dual (line 
\ref{line:5a}) updates of Algorithm \ref{alg:alg1} involve terms to  compute the violations associated to constraints \eqref{eq:fair_lern_b}. These terms rely on the attributes $A$, and therefore, the resulting model leaks the sensitive information. 
To contrast this issue, this section introduces an extension to F-LD, 
called \emph{Private and Fair Lagrangian Dual (PF-LD)} method, that 
guarantees both fairness and privacy. The idea is to render the 
computations of the primal and dual update steps differentially 
private with respect to the sensitive attributes.

\subsubsection*{Private Primal Update} 
\label{sub:private_primal_step}
At each epoch $k$, the primal update (line \ref{line:4a} of Algorithm 
\ref{alg:alg1}) 
computes the gradients over the loss function $\cL_{\bm{\lambda}_k}(\theta)$, which is composed 
of two terms (see Equation \eqref{eq:lagrangian_function}). 
The first term, $J(\cM_\theta, D_P)$, uses exclusively public information, 
while the second term, 
$\bm{\lambda}^\top |\bm{\mu}(D_P) - \bm{\mu}(D_G) |$ 
requires both the public and sensitive group information. 
The computation of these gradients can be made differentially private by the introduction of carefully calibrated Gaussian noise.
The general concept, relies on performing a \emph{differentially private 
Stochastic Gradient Descent (DP-SDG)} step \cite{abadi:16}. In a nutshell, 
DP-SDG computes the gradients for each data sample in a random mini-batch, 
clips their L2-norm, computes the average, and adds noise to ensure privacy.

The result below bounds the global sensitivity $\Delta_p$ of the 
\emph{sensitive} term in the primal update, which is needed to calibrate the noise necessary to guarantee privacy. 

\begin{theorem}
\label{thm:private_primal}
Let \( \| \nabla_\theta h(z)\| \!\leq\! C_p\), for all $z \!\in\! B_{G_i}$, $i \!\in\! \cI$, and some $C_p \!>\! 0$.
The global sensitivity $\Delta_p$ of the gradients of the constraints violation 
$\nabla_{\theta}  \bm{\lambda}^\top |\bm{\mu}(B_P) - \bm{\mu}(B_G)|$ is
\begin{small}
	\begin{equation}
	\label{eq:sensitivity_Delta_p}
	    \Delta_p 
	    \leq \frac{2 C_p \lambda^{\max}}{ 
	    \min_{i \in \cI} 
	    |B_{G_i}|-1}.
	\end{equation}
\end{small}
\end{theorem}

 The above uses a clipping term, $C_p$, to control the maximal change of the gradients. 
Crucially, this is non-limiting, as it can be enforced by clipping 
the gradient contribution $\left\| \nabla_{\theta} h(z) \right\|$ to $C_p$, similarly to what done in DP-SDG.

\begin{proof}
Consider two neighboring dataset $B$ and $B'$ differing in the 
membership to a protected group of one participating sample $z =(X,A,Y)$.
W.l.o.g., consider a sample $z$ that changes membership from group 
$A=k_1$ to $A'=k_2$. Upon this change, there are exactly two groups 
($B_{G_{k_1}}$ and $B_{G_{k_2}}$) whose sizes are being affected. 
Namely, one group size increases while the other decrease.
Additionally notice that $|B_{G_i} - B_{G_i}'| \leq 1$ when 
$i \in \{k_1, k_2\}$. 
Additionally, notice that this change does not impact the 
\emph{population terms} $B_P$ and $B_P'$. 
For $i \in \{k_1, k_2\}$, the gradient contributions
of the constraint violations associated to the group terms can be bound
as:
\begin{subequations}
\label{a:eq:p0}
\begin{align} 
&\left\| 
    \nabla_\theta \lambda_i 
    \left|  \mu(B_{P_i}) - \mu(B_{G_i}) \right|
    -
    \nabla_\theta \lambda_i 
    \left|  \mu(B'_{P_i}) - \mu(B'_{G_i}) \right|
\right\| \\
=&
\left\| 
    \nabla_\theta \lambda_i 
    \left|  \mu(B'_{G_i}) - \mu(B_{G_i}) \right|
\right\| \label{a:eq:p1} \\
=&
\lambda_i 
\left\| 
    \nabla_\theta \hat{\EE}_{z \sim B'_{G_i}} \left[h(z)\right] -
    \nabla_\theta \hat{\EE}_{z \sim B_{G_i}} \left[h(z)\right]
\right\| \label{a:eq:p2} \\
=&
\lambda_i 
\left\| 
    \nabla_\theta \frac{1}{|B'_{G_i}|} \sum_{z \in B'_{G_i}} h(z) -
    \nabla_\theta \frac{1}{|B_{G_i}|} \sum_{z \in B_{G_i}} h(z) 
\right\| \\
\leq&
\lambda_i 
\left\| 
    \nabla_\theta \frac{1}{|B_{G_i}|-1} h(z)
\right\| \label{a:eq:p3} \\
=&
\frac{\lambda_i}{|B_{G_i}|-1}
\left\| 
    \nabla_\theta h(z)
\right\| 
\leq 
\frac{\lambda_i C_p}{|B_{G_i}|-1}
\label{a:eq:p4}
\end{align}
\end{subequations}
where equation \eqref{a:eq:p1} follows from that $B_P = B_P'$ (and thus 
$\mu(B_{P_i}) = \mu(B_{P'_i})$) and since the whole expression is under
a norm operator, 
equation \eqref{a:eq:p2} from that $\lambda_i \geq 0$ and from
definition of the $\mu$ terms, 
equation \eqref{a:eq:p3} follows from the notion of adjacent datasets 
and that there is a single element differing between $B_{G_i}$ and 
$B'_{G_i}$, and, finally, equation \eqref{a:eq:p4} follows from the 
theorem assumption. 
Therefore, the global sensitivity of the gradient contributions of the 
constraint violations can be bounded above as:
\begin{subequations}
\label{a:eq:p5}
\begin{align} 
\Delta_p &= \max_{B,B'}
    \left\| 
    \nabla_{\theta}  \bm{\lambda}^\top |\bm{\mu}(B_P) - \bm{\mu}(B_G)|
    -
    \nabla_{\theta} \bm{\lambda}^\top |\bm{\mu}(B_P') - \bm{\mu}(B_G')|
    \right\|\\
&\leq
\sum_{i \in \cI}
\left\| 
\nabla_{\theta}  \lambda_i |\mu(B_{P_i}) - \mu(B_{G_i})|
-
\nabla_{\theta} \lambda_i |\mu(B_{P_i}') - \mu(B_{G_i}')|
\right\|\\
&\leq
    \frac{C_p \lambda_{k_1}}{|B_{G_{k_1}}|-1} + 
    \frac{C_p \lambda_{k_2}}{|B_{G_{k_2}}|-1}
\leq
    \frac{2 C_p \lambda^{\max}}{ 
    \min_{i \in \cI} |B_{G_i}|-1}.
\end{align}
\end{subequations}
where the last inequality follows from Equation \eqref{a:eq:p0}, and 
noticing that (1) there are exactly two groups $B_{G_i}$ 
($i \in \{k_1, k_2\}$) whose size is being affected, 
(2) for any $i\in \cI$, $\lambda_{i} \leq \lambda^{\max}$ 
and (3) $|B_{G_{i}}| \geq \min_{i \in \cI} |B_{G_i}|$.
\end{proof}

Using Theorem \ref{thm:private_primal}, the privacy-preserving primal update step for a mini-batch $B \!\subseteq\! D$ can be executed by clipping exclusively the gradients of the functions 
$h(z)$ associated with the group terms in $B_G$. \emph{It is not necessary to perform gradient clipping for the functions $h(z)$ associated with the population terms in $B_P$}. 
While this may induce propagating population and group terms gradients of different magnitudes, the authors observed often improved performance in the adopted setting. 
Thus, PF-LD substitutes line \ref{line:4a} of Algorithm \ref{alg:alg1} with the following
{
\begin{align}
\label{eqn:private_primal_update}
	\theta \gets \theta - \alpha 
	\big( \nabla_\theta \left[ J(\cM_\theta, B_P)  \right] +
  	\bm{\lambda}^\top 
  	\left| 
  	\nabla_\theta \bm{\mu}(B_P) - \bar{\nabla}^{C_p}_\theta 
  	\bm{\mu}(B_G) \right|
	+ \cN(0, \sigma_p^2\ \Delta^2_p \bm{I}) 
	\!\big), 
\end{align}
}
\noindent
with $\bm{I} \!\in\! \{0,1\}^{|\cI| \times |\cI|}$, $\sigma_p \!>\! 0$, and 
$\bar{\nabla}^{C_p}_\theta$ is applied to each element $\mu(B_{G_i})$
of vector $\bm{\mu}(B_G)$, where 
\[
    \bar{\nabla}_\theta^{C_p}(x) \!=\! \frac{\nabla x}{\max (1, \frac{\|\nabla x\|}{C_p})}
\]
denotes the gradients of a given scalar loss $x$ clipped in a $C_p$-ball, for $C_p > 0$.

\subsubsection*{Private Dual Update} 
\label{sub:private_dual_step}
Similar to the primal step, the dual update requires access to the 
sensitive group information (see line \ref{line:5a} of Algorithm 
\ref{alg:alg1}). It updates the multipliers based on amount of constraint 
violation 
\(
	\left|
		\bm{\mu}(D_P) - \bm{\mu}(D_G) 
	\right|
\)
computed over the entire dataset $D$. Privacy can be attained by injecting Gaussian 
noise to the computation of the multipliers, but computing the global 
sensitivity $\Delta_d$ of the constraint violations is non-trivial 
since the range of the violations is unbounded. Once again, the paper 
recurs to the adoption of a clipping term, $C_d$, that controls 
the maximal contribution of the constraint violation to the associated 
multiplier value. 
\begin{theorem}
\label{thm:private_dual}
    Let \( |h(z)| \!\leq\! C_d\), for all samples $z \in D_{G_i}$, $i \!\in\! \cI$, and some $C_d \!>\! 0$. The global sensitivity 
    $\Delta_d$ of the constraint violation 
	$\left| \bm{\mu}(D_P) - \bm{\mu}(D_G)\right|$ is
    \begin{align}
        \label{eq:sensitivity_Delta_d}
        \Delta_d 
        &\leq \frac{\sqrt{2} C_d}{\min_{i \in \cI} |D_{G_i}|-1}.
    \end{align}
\end{theorem}

\begin{proof}
Consider two neighboring datasets $D$ and $D'$ 
differing by the group membership of one particular sample $z$.  
Using a similar argument as that used in the proof of Theorem 
\ref{thm:private_primal}, consider a sample $z$ that changes membership 
from group $A=k_1$ to $A'=k_2$. Upon this change, there are exactly 
two groups ($B_{G_{k_1}}$ and $B_{G_{k_2}}$) whose sizes are being 
affected. Namely, one group size increases while the other decrease.
Additionally notice that $|B_{G_i} - B_{G_i}'| \leq 1$ when 
$i \in \{k_1, k_2\}$. Additionally, notice that this change does not 
impact the \emph{population terms} $B_P$ and $B_P'$. 

The maximal amount of difference between the conditional empirical 
mean $\mu(D_{G_i})$ constructed from $D$ and $\mu(D'_{G_i})$ constructed 
from $D'$, for any given $D, D'$, can be bounded as
\begin{subequations}
\label{a:eq:p2_0}
\begin{align} 
&\left\| 
    \left|  \bm{\mu}(D_{P}) - \bm{\mu}(D_{G}) \right|
    -
    \left|  \bm{\mu}(D'_{P}) - \bm{\mu}(D'_{G}) \right|
\right\| \label{a:eq:p2_1}\\
\leq&
\sum_{i \in \cI}
\left\| 
    \left|  \mu(D_{P_i}) - \mu(D_{G_i}) \right|
    -
    \left|  \mu(D'_{P_i}) - \mu(D'_{G_i}) \right|
\right\| \label{a:eq:p2_2}\\
=&
\sum_{i \in \cI}
\left\| 
    \left| \mu(D'_{G_i}) - \mu(D_{G_i}) \right|
\right\| \label{a:eq:p2_3} \\
=&
\sum_{i \in \cI}
\left\| 
    \left|\hat{\EE}_{z \sim D'_{G_i}} [h(z) ] -  
          \hat{\EE}_{z \sim D_{G_i}}[h(z)] \right| 
\right| \label{a:eq:p2_4}\\
=  &
\sum_{i \in \cI}
\left\|
    \left| \frac{1}{|D'_{G_i}|}\sum_{z \in D'_{G_i}} h(z)  
         - \frac{1}{|D_{G_i}|} \sum_{z \in D_{G_i} } h(z)\right| 
\right\| \label{a:eq:p2_5}\\ 
\leq &
\left\|
    \frac{|h(z)|}{|D_{G_{k_1}}| - 1} + \frac{|h(z)|}{|D_{G_{k_2}}| - 1}
\right\| \label{a:eq:p2_6}\\
\leq &
\left\|
    \frac{C_d}{|D_{G_{k_1}}| - 1} + \frac{C_d}{|D_{G_{k_2}}| - 1}
\right\| \label{a:eq:p2_7}\\
=&
    \sqrt{\left(\frac{C_d}{|D_{G_{k_1}}|-1}\right)^2 + \left(\frac{C_d}{|D_{G_{k_2}}|-1}\right)^2}
\label{a:eq:p2_8}\\
\leq& \frac{\sqrt{2} C_d}{\min_{i \in \cI} |D_{G_i}-1|}
\label{a:eq:p2_9}
\end{align}
\end{subequations}
where \eqref{a:eq:p2_2} follows from triangle inequality, 
\eqref{a:eq:p2_3} from that $D_{P_i} = D'_{P_i}$, for any $i \in \cI$
(and thus $\mu(B_{P_i}) = \mu(B_{P'_i})$) and since the whole expression 
is under a norm operator, 
\eqref{a:eq:p2_4} and \eqref{a:eq:p2_5} follow by definition of the 
function $\mu$. 
The inequality \eqref{a:eq:p2_6} follows from that $D_{G_i}$ and $D_{G_i}'$ 
differ in at most one element (when $i = k_1$ or $i = k_2$) and
\eqref{a:eq:p2_7}, by assumption \(|h(z)|\leq C_d\), and, finally,
\eqref{a:eq:p2_9} follows by noting that 
$|D_{G_i}| \geq \min_{i \in \cI} |D_{G_i}|$.
\end{proof}

The privacy-preserving dual update step, used in lieu of line \ref{line:5a} of Algorithm \ref{alg:alg1}, is given by the following
{
\begin{align}
    \label{eqn:private_dual_update}
   \!\!\!\bm{\lambda}_{k+1} \gets  \bm{\lambda}_{k}  &+ s_k \left(
   | \bm{\mu}(D_P) - \bar{\bm{\mu}}^{C_d}(D_G)|\!+\! 
   \cN(0, \sigma_d^2\Delta_d^2 \bm{I}) \right)\!\!
\end{align}
}
with $\bm{I} \in \{0,1\}^{|\cI| \times |\cI|}$, $\sigma_d > 0$, and where, for every $i \in \cI$,
{
\begin{equation*}
    \bar{\mu}^{C_d}(D_{G_i}) = 
    \hat{\EE}_{z \sim D_{G_i}} 
    \left[\frac{h(z)}{\max(1, \frac{|h(z)|}{C_d})} 
	\right].
\end{equation*}
}
Note that while Theorem \ref{thm:private_primal} bounds the individual gradient norms of each functions $h(z)$ for samples $z \!\in\! B_{G_i}$ and $i \!\in\! \cI$, Theorem \ref{thm:private_dual} bounds their maximum absolute values. 
The choice of terms $C_p$ and $C_d$ plays a special role 
in limiting the impact of an individual change in the protected attributes. 
It controls, indirectly, the privacy loss, as it impacts the global 
sensitivities $\Delta_p$ and $\Delta_d$. However, these terms also 
affect the model accuracy and fairness. 
In particular, larger $C_p$ values will propagate more precise gradients ensuring better accuracy and model fairness. 
On the other hand, larger clipping values will also introduce more  noise, and thus degrade the information propagated. The converse is true for smaller clipping values.
A theoretical and experimental analysis of the impact of these terms to the model accuracy and fairness is provided in the next sections.

\subsection{Privacy Analysis}
\label{sec:privacy_analysis}

The privacy analysis of PF-LD relies on the moment accountant for Sampled 
Gaussian (SG) mechanism \cite{mironov2019rnyi}, whose privacy is analyzed 
using \emph{R\'{e}nyi Differential Privacy} (RDP) \cite{Mironov_2017}. 

\begin{definition}
\label{definition7}
[Sampled Gaussian Mechanism]
Let $f:S\subseteq D \to \RR^d$ be a function mapping subsets $S$ of the input data $D$ to $\RR^d$. 
The Sampled Gaussian (SG) mechanism with sampling rate $0 < q \leq 1 $ and standard deviation $\sigma > 0$ is defined as follows:
{
\begin{align*}
SG_{q,\sigma}(D) \triangleq 
f(\{x: x \in D \  \mbox{is sampled with probability q}  \})
+ \cN(0, \sigma^2 \bm{I}),
\end{align*}
}
where each element of $D$ is sampled independently at random without replacement with probability $q$, and $\cN(0, \sigma^2 \bm{I})$ is the spherical $d$-dimensional Gaussian noise with per-coordinate variance $\sigma^2$.
\end{definition}

\begin{theorem}(($\alpha,\epsilon$)-RDP)
\label{thm:RDP}
A randomized mechanism $f:\cD \to \cR$ with domain $\cD$ and range $\cR$ is said to have $\epsilon$-R\'{e}nyi differential privacy of order $\alpha$, or ($\alpha,\epsilon$)-RDP for short, if for any adjacent $D,D' \in \cD$ it holds that
\[
    \cD_\alpha(f(D) \parallel f(D')) \leq \epsilon,
\]
where 
$\cD_{\alpha} (P \parallel Q) \triangleq \frac{1}{1-\alpha} 
\log \EE_{x \sim Q} \left( \frac{P(x)}{Q(x)}\right)^{\alpha}$ is the R\'{e}nyi divergence of order $\alpha$ between two probability distributions $P$ and $Q$. 
\end{theorem}

The privacy analysis of the SG mechanism is described by the following Theorem from \cite{mironov2019rnyi}.
\begin{theorem}
\label{thm:RDP_sampling}
For a given function $f: D \to \RR^d$, such that $||f(D) - f(D') ||_2 \leq 1$ for any neighboring databases $D, D'$, the SG mechanism $SG_{q,\sigma}$ with sampling ratio $q$ and standard deviation $\sigma$ satisfies $(\alpha, \epsilon)$-RDP with:
{
\begin{subequations}
 \label{eq_app:RDP}
    \begin{align}
    \epsilon  & \leq \cD_{\alpha} 
        \left[ \cN (0, \sigma^2 ) \;\parallel\; (1-q) \cN (0, \sigma^2 ) + q \cN(1, \sigma^2 )  
        \right] \label{eq_app:RDPa}\\
    \epsilon & \leq \cD_{\alpha} 
        \left[ (1-q) \cN (0, \sigma^2 ) 
                + q \cN (1, \sigma^2 ) \;\parallel\; 
                \mathcal{N} (0, \sigma^2 ) 
        \right]. \label{eq_app:RDPb}
    \end{align}
\end{subequations}
}
\end{theorem}
The R\'{e}nyi divergences appearing in Equations \eqref{eq_app:RDPa} and \eqref{eq_app:RDPb} can be computed numerically according to the procedure described in \cite{mironov2019rnyi}. 
Additionally, RDP enjoys composition: if $\cM_1$ and $\cM_2$ are mechanisms satisfying $(\alpha, \epsilon_1)$-RDP and $(\alpha, \epsilon_2)$-RDP, respectively, then their composition satisfies $(\alpha, \epsilon_1 + \epsilon_2)$-RDP \cite{mironov2017renyi}

The following results consider a PF-LD algorithm that trains a model over $T$ epochs using a dataset $D$ containing $n$ training samples, uses mini-batches $B$ at each iteration, and standard deviation parameters $\sigma_p$ and $\sigma_d$, associated to the primal and dual steps, respectively. Note that, Equations \eqref{eqn:private_primal_update} and \eqref{eqn:private_dual_update}, correspond to instances of the SG mechanism. 
\begin{lemma}
	\label{lm:dp_primal}
	The PF-LD primal step satisfies $(\alpha, \epsilon_p)$-RDP, where 
	$\epsilon_p$ satisfies Equations \eqref{eq_app:RDP} 
	with $q \!=\! \nicefrac{|B|}{n}$ and $\sigma=\sigma_p\Delta_p$ are, respectively, the SG sampling ratio $q$ and standard deviation parameters. 
\end{lemma}
\begin{proof}
The result follows directly by Definition \ref{definition7} and Theorems \ref{thm:RDP} and \ref{thm:RDP_sampling} since the primal step of PF-LD uses Gaussian noise with parameter $\sigma_p$ over a subsample of the dataset of size $|B|$.
\end{proof}

\begin{lemma}
	\label{lm:dp_dual}
	The PF-LD dual step satisfies $(\alpha, \epsilon_d)$-RDP, where $\epsilon_d$ satisfies Equations \eqref{eq_app:RDP} 
	with $q \!=\! 1$ and $\sigma\!=\!\sigma_d\Delta_d$ are, respectively, the SG sampling ratio and standard deviation parameters. 
\end{lemma}
\begin{proof}
The result follows directly by Definition \ref{definition7} and Theorem \ref{thm:RDP} since the dual step of PF-LD uses Gaussian noise with parameter $\sigma_d$.
\end{proof}

PF-LD uses a predefined amount of noise (specified by parameters $\sigma_p$ and $\sigma_d$) at each iteration, so that each iteration  has roughly the same privacy loss, and uses the moment accountant \cite{abadi:16} to track detailed information of the cumulative privacy loss.
\begin{theorem}
	\label{thm:dp_PFLD}
	 PF-LD satisfies $(\alpha, \frac{Tn\epsilon_{p}}{|B|}+T\epsilon_d)$-RDP. 
\end{theorem}

\begin{proof}
    The result follows from Lemmas \ref{lm:dp_primal} and \ref{lm:dp_dual} and by the composability of RDP \cite{mironov2017renyi}.
\end{proof}

The final privacy loss in the $(\epsilon, \delta)$-differential privacy model is obtained by observing that a mechanism satisfying $(\alpha, \epsilon)$-R\'{e}nyi differential privacy also satisfies $(\epsilon + \frac{\log \nicefrac{1}{\delta}}{\alpha-1}, \delta)$-differential privacy, for any $0\!<\! \delta \!<\! 1$ \cite{mironov2017renyi}.

\subsection{Bias-Variance Analysis}
\label{sec:bias-variance}

A key aspect of PF-LD is the choice of values $C_d$ and $C_p$ used 
to bound the functions $h(z)$, and their gradients, 
respectively, for every sample $z \in D_{G_i}$ and $i \in \cI$. 
The choice of these clipping terms affects the global sensitivity of 
the functions of interest, which, in turn, impacts the amount of 
noise used by the differentially private mechanisms. As a result, 
the clipping terms are associated to a \emph{bias-variance} 
trade-off: 
Small values can discard  significant amounts of \emph{constraints 
information}, thus introduce bias; large values retain more 
information but force the 
differential privacy mechanism to introduce larger noise, inducing 
more variance. 
It is important to recall that, at every iteration, PF-LD induces 
some privacy loss, thus, for a fixed privacy budget, the use of 
small values cannot be compensated by longer runs. 
This section formulates a bias-variance analysis that is helpful to 
select clipping values for gradient norms under the SG mechanism. 
 
\def\true{\textsl{true}}
\def\priv{\textsl{priv}}
Let $\bm{G} \!=\! \nabla_{\theta} \bm{\lambda}^{\top}\! 
		|\bm{\mu}(B_P) - \bm{\mu}(B_G)|$ 
be the gradient computed over a minibatch $B$ during the primal update of F-LD (Algorithm \ref{alg:alg1} line \ref{line:4a}) and 
$\tilde{\bm{G}} \!=\! 
  	\bm{\lambda}^\top\!
  	| \nabla_\theta \bm{\mu}(B_P) -
  	\bar{\nabla}^{C_p}_\theta \bm{\mu}(B_G) |
	+ \cN(0, \sigma_p^2\Delta^2_p \bm{I})
$
be its privacy-preserving counterpart, as computed by PF-LD  (Equation \eqref{eqn:private_primal_update}).
\begin{theorem} 
\label{thm:Cp_bound}
The expected error between the real and noisy gradients, $\bm{G}$ and $\tilde{\bm{G}}$, incurred during the primal step can be upper bounded as: 
\begin{align}
\label{eq:grad_error} 
     \EE\left[ \big\|\bm{G} - \tilde{\bm{G}}\big\| \right] 
     \leq & \frac{2 \sqrt{S_{\tilde{\bm{G}}}} \sigma_p \lambda^{\max}C_p}
     {\min_{i \in \cI} |B_{G_i}|-1} + \sum_{i \in \cI} \lambda_i \hat{\EE}_{z\sim B_{G_i}} 
      \left[ \max\big(0, \left\|\nabla_\theta h(z) \right\|-C_p\big) 
      \right],\notag
\end{align}
where $S_{\tilde{\bm{G}}}$ is the shape (i.e., the number of entries) of 
$\tilde{\bm{G}}$.
\end{theorem}
\begin{proof}
    By triangular inequality,
\[
      \EE\left[ \|\bm{G} - \tilde{\bm{G}}\| \right] \leq  
      \EE\left[ 
      \|\tilde{\bm{G}} - \EE[\tilde{\bm{G}}]\| 
      \right]
      +  
      \|\EE[\tilde{\bm{G}}] - \bm{G}\|
\]
Where the first term of the right hand side inequality represents the variance term and the second is the bias term. 
The squared variance term can be bounded as follows:
\begin{align}
    \left( 
    \EE\left[ \|\tilde{\bm{G}} - \EE[\tilde{\bm{G}}]\| \right]
    \right)^2 
    \leq 
    \EE\left[\|\tilde{\bm{G}} -  \EE[\tilde{\bm{G}}]\|^2 \right] 
    = S_{\tilde{\bm{G}}} \sigma^2_p \Delta^2_p.
    \label{a:eqn:variance_bound}
\end{align}
 The above follows from Jensen inequality: For any convex function $f$, $f(\EE[X]) \leq  \EE[f(X)]$, and by setting $f=(\cdot)^2$ and 
 $X = \| \tilde{\bm{G}} - \EE[\tilde{\bm{G}}] \|$, 
 and from noticing that the term 
 $\EE\left[\|\tilde{\bm{G}} -  \EE[\tilde{\bm{G}}]\|^2 \right]$
 correspond to the variance of the Gaussian noise added to 
 the each of the $S_{\tilde{\bm{G}}}$ model gradients (see Equation \eqref{eqn:private_primal_update}).  
 Therefore $ \EE \left[ \| \tilde{\bm{G}} - \EE [\tilde{\bm{G}}] \| \right] 
            \leq   \sqrt{S_{\tilde{\bm{G}}}} \sigma_p \Delta_p $. 
The bias term can be bound as follows:
\begin{align}
    \left\| \EE[\tilde{\bm{G}}]  - \bm{G} \right\|  
    &= \left\| \sum_{i \in \cI} \lambda_i 
        \left(   \nabla_\theta \big| \mu(B_{P_i}) - \bar{\nabla}^{C_p}_\theta 
    \mu (B_{G_i})\big| -  \nabla_\theta \big| \mu(B_{P_i}) - \nabla_\theta 
    \mu(B_{G_i}) \big|\right) \right\|  \nonumber \\
    &= \left\| \sum_{i \in \cI} \lambda_i 
        \left( \nabla_\theta \mu(B_{G_i})  
            - \bar{\nabla}^{C_p}_\theta \mu(B_{G_i})  
            |\right) \right\| \nonumber \\
    &=   \sum_{i \in \cI} \lambda_i \left\|
        \hat{\EE}_{z \sim B_{G_i}} \left[\nabla_\theta h(z)\right]
      - \hat{\EE}_{z \sim B_{G_i}} \left[\bar{\nabla}_\theta^{C_p} h(z)\right] 
        \right\|  \nonumber \\
    &\leq 
    \sum_{i \in \cI} \lambda_i \hat{\EE}_{z \sim B_{G_i}} 
    \left[ \left\| 
    \nabla_\theta h(z) - \bar{\nabla}_\theta^{C_p}h(z)
    \right\| \right] 
    \tag{by Jensen inequality} \\
    & =  \sum_{i \in \cI} \lambda_i \hat{\EE}_{z \sim B_{G_i}} \left[  \max\big( 0, \left\|\nabla_\theta h(z) \right\|-C_p \big)  \right],
    \label{a:eqn:bias_bound}
\end{align}
where the last equality is due to gradient clipping.

Combining Equation \eqref{a:eqn:variance_bound} with Equation \eqref{a:eqn:bias_bound} and replacing the term $\Delta_p$ with $\frac{2 C_p \lambda^{\max}}{ \min_{i \in \cI} |B_{G_i}| -1}$ (by Theorem \ref{thm:private_primal}), gives the sought upper bound for the expected error of the private gradients.
\end{proof}

Note that the bound above is a convex function of $C_p$. Its unique minimizer 
satisfies:
{
\begin{equation*}
    \frac{2\sqrt{S_{\tilde{\bm{G}}}}\sigma_p \lambda^{\max}}
    	 {\min_{i \in \cI} |B_{G_i}|-1} 
    \!=\!  
    \sum_{i\in\cI} \!\lambda_i \hat{\EE}_{z \sim B_{G_i}}\! 
    \Big[ \mathbbm{1} \big[ \| \nabla_\theta h(z)\| \geq C_p\big]\Big].
\end{equation*}
}
While beyond the scope of the this work, the above illustrates that a procedure to find the optimal $C_p$ privately can be constructed effectively. 

Next, the paper shows how to bound the expected error incurred in using the noisy constraint violations during the dual step. 
Let  $V_i \!=\! | \mu(D_{P_i}) - \mu(D_{G_i}) |$ be the value corresponding to the i-th constraint violation ($i \in \cI$), and  
$\tilde{V}_i \!=\! |\mu(D_{P_i}) - \bar{\mu}^{C_d}(D_{G_i})| + \cN(0, \sigma_d^2 \Delta_d^2) $ be its privacy-preserving version (see Equation \eqref{eqn:private_dual_update}). 
\begin{theorem} 
\label{thm:Cd_bound}
The expected absolute error between the real and noisy constraint violations $V_i$ and $\tilde{V}_i$, for $i \!\in\! \cI$, is bounded by the following 
{
\begin{align}
\label{eq:constraint_error} 
 \EE \left[ |V_i- \tilde{V}_i | \right] &\leq
         \frac{\sqrt{2} C_d \, \sigma_d}{\min_{i \in \cI} 
         |D_{G_i}|-1 } 
          + \hat{\EE}_{z \sim D_{G_i}}\left[ \max(0, |h(z)| - C_d) \right]. \notag 
\end{align}
}
\end{theorem}

\noindent The proof uses similar arguments as those in the proof of 
Theorem \ref{thm:Cp_bound} 

\begin{proof}
 By triangular  inequality,
\begin{align*}
      \EE\left[ |V_i - \tilde{V}_i| \right] 
      \leq 
        \left|V_i - \EE[\tilde{V}_i] \right| 
      + \EE\left[ 
        \left| \tilde{V}_i - \EE \left[ \tilde{V}_i \right] \right| 
      \right]
\end{align*}
Where the first term of the right hand side inequality represents the 
variance term and the second is the bias term. 

Using the Jensen inequality, the squared variance term can be bound 
as follows:
$$
    \left( \EE\left[ | \tilde{V}_i - \EE[\tilde{V}_i ] |\right]\right)^2  
    \leq \EE\left[ \left(\tilde{V}_i - \EE[\tilde{V}_i] \right)^2 \right]  
   = \sigma^2_d \Delta^2_d, 
$$
where the last equality holds because the noisy constraint violations
are perturbed with Gaussian random noise with a standard deviation  
$\sigma_d \Delta_d$. 
By replacing 
$\Delta_d = \frac{\sqrt{2}C_d}{\min_{i \in \cI} |D_{G_i}|}$ 
by Theorem \ref{thm:private_primal}, 
the variance term now can thus be upper bounded by:
\begin{equation}
\label{a:eq:variance_term_bound}
     \EE \left[ | \tilde{V}_i - \EE[\tilde{V}_i ]  | \right] 
     \leq  \frac{\sqrt{2}  C_d \, \sigma_d}{  \min_{i \in \cI } |D_{G_i} | -1 }.
\end{equation}

Next, focusing on the bias term, note that the true constraint $V_i$ is 
represented  as: 
\begin{align}
\label{a:eq:vi}
    V_i &= |\mu(D_{P_i}) - \mu(D_{G_i})|
        = \left|
            \hat{\EE}_{z \sim D_{P_i}}[h(z)] 
          - \hat{\EE}_{z \sim D_{G_i}}[h(z)]  
          \right|,
\end{align}
and the expectation of its private version, $\tilde{V}_i$, can be 
written as
\begin{align}
\label{a:eq:E_vi}
    \EE\left[\tilde{V}_i\right] &= 
    \left| \hat{\EE}_{z \sim D_{P_i}} \left[h(z))  \right] 
         - \hat{\EE}_{z \sim D_{G_i}} \left[ \frac{h(z)}
        {\max(1, \frac{|h(z)|}{ C_d})} \right]
    \right|
\end{align}
where the above follows from the definition of the constraint on the 
group term $\bar{\mu}^{C_d}(D_{G_i}$ (see main paper, equation after 
Equation \eqref{eqn:private_dual_update}).
Combining Equations \eqref{a:eq:vi} and \eqref{a:eq:E_vi}, the bias 
term can be rewritten as:
\begin{align}
    \left| V_i - \EE[\tilde{V}_i] \right| 
 &= \left| \hat{\EE}_{z \sim D_{G_i}} \left[ h(z)  \right] - \hat{\EE}_{z \sim D_{G_i}} \left[ \frac{h(z) }
    {\max(1, \frac{|h(z)|}{ C_d})} \ \right] \right| \notag\\
 & = \EE_{z \sim D_{G_i}}\left[ \max(0, |h(z))| - C_d) | \right].
    \label{a:eqn:bias_term_bound}
\end{align} 

Combining Equation \eqref{a:eqn:bias_term_bound} and 
\eqref{a:eq:variance_term_bound} provides the sought bound.
\end{proof}

\section{PF-LD: Handing Missing Values}
\label{sec:missing_values}

This section presents a simple, yet important, extension to the PF-LD model where only a subset of the individuals reports the value of the sensitive attributes. 
Recall that to build a fair model it is necessary to collect sensitive information about users. While the privacy risks of this data collection can be mitigated with the techniques proposed in this paper, a practical scenario may involve the data curator to give the user a choice to whether release their sensitive information or not. 
This section shows that it is straightforward to adapt the proposed model to this scenario.

The paper 
considers the case where only a fraction $r \leq 1$ of the training samples presents the sensitive attribute $A$. 
To operate under this scenario, it is sufficient to modify the sensitivity $\Delta_p$ of the gradients of the constraint violations (Equation \eqref{eq:sensitivity_Delta_p}) and the sensitivity $\Delta_d$ of the constraint violations (Equation \eqref{eq:sensitivity_Delta_d})
to, respectively, $\Delta_p' = \nicefrac{\Delta_p}{r}$ 
and $\Delta_d' = \nicefrac{\Delta_d}{r}$. 
The argument follows from the observation that reducing the number of training samples having sensitive data also reduces the components $\min_{i\in \cI}|B_{G_i}|$ (Equation \eqref{eq:sensitivity_Delta_p}) 
and $\min_{i \in \cI}|D_{G_i}|$ (Equation \eqref{eq:sensitivity_Delta_d}) by a factor of $r$. 
Consequentially, while on one side, using less information may affect the model ability to satisfy the fairness constraint, on the other, it also require smaller amounts of noise to guarantee the same level of privacy. This trade-off is subject of analytical study, presented next.

\newcommand{\algAf}{\textit{Agarwal et.~al}}
\newcommand{\algMf}{\textit{Mozannar et.~al}}
\newcommand{\algZf}{\textit {Zafar et.~al}}
\newcommand{\algA}{\texttt{A}}
\newcommand{\algM}{\texttt{M}}
\newcommand{\algZ}{\texttt{Z}}
\section{Experimental Analysis}
\label{sec:experiment}

\paragraph{Datasets, Models, and Metrics} This section studies the behavior of the proposed algorithm on several datasets, including \emph{Income}. \emph{Bank}, and \emph{Compass} \cite{zafar2017fairnes} datasets, described as follows.

\begin{itemize}

    \item \textit{Income}:
    The task is to predict if an user has annual income  
    above \$50,000. The protected attribute associated with the group 
    membership is the \emph{gender}, and there are two groups: 
    \emph{male} and \emph{female}.

    \item \textit{Compas}: 
    The task is to predict if a defendant will re-offend in the next 
    two years (2 years recidivism). The protected attribute associated 
    with the group membership is \emph{gender}.

    \item \textit{Bank}:
    The task is to detect client subscriptions to the term deposit. 
    The protected attribute associated with the group membership is 
    \emph{age} and the experiments consider the following three different group membership sizes: 
    \begin{itemize}
        \item \textit{Bank}, in which the number of protected groups 
        $|\cA| = 2$,  with the majority of the people of age ranging from 25 to 60 and 
        minority, being under 25 or over 60 years old. 
        \item $\mbox{M-Bank}^3$, in which the number of protected 
        groups is $|\cA| = 3$, indicating people with age ranging from 25 to 40, 
        41 to 60 and the remaining ages.
        \item $\mbox{M-Bank}^5$, in which the number of protected groups 
        is $|\cA| = 5$, indicating people with age ranging from 25 to 33, 
        34 to 40, 41 to 48, 49 to 60, and the remaining ages.
    \end{itemize}
\end{itemize}
All datasets above were pre-processed according to 
\cite{zafar2015fairness, zafar2017fairnes} and are so that all 
features have zero mean and unit variance.
A summary of the datasets adopted and their features is provided in 
Table~\ref{tab:dataset_summary}.

\begin{table*}[!tb]
\centering
\resizebox{1.0\textwidth}{!}{%
\begin{tabular}{r*{6}{l}}
\toprule
  Data  &  \# Samples & \# Features & $\cY$ & $\cA$&   $|\cA|$& Protected group sizes (\%) \\
\midrule
Bank      & 11,162              & 15   &  Deposit            & Age     &2 & 8\% ; 92\%        \\
Income   & 45,222               & 50    & $\geq$ 50K         & Gender  &2 & 32\% ; 68\%      \\
Compas    & 6,150               & 10    & 2 years recidivism & Gender  &2 & 20\% ; 80\%       \\ 
$\mbox{M-Bank}^3$ &   11,162    & 15    & Deposit            & Age     &3 & 8\% ; 46\% ; 46\% \\
$\mbox{M-Bank}^5$ &   11,162    & 15    & Deposit            & Age     &5 & 8\% ; 23\% ; 23\% ; 23\% ; 23\%\\
\bottomrule
\end{tabular}
}
\caption{Dataset summary}
\label{tab:dataset_summary}
\end{table*}

The experiments consider a baseline classifier (\textit{CLF}), implemented as a neural network with two hidden layers, that maximize accuracy only, without considerations for fairness or privacy, and compare the proposed \textit{PF-LD} model against the following state-of-the-art algorithms: 
    \algZ, it implements a fair logistic regression models that achieves group fairness. These models were presented in \cite{zafar2015fairness} for demographic parity and in \cite{zafar2017fairnes} for accuracy parity and equalized odds. 
    \algA, it implements the fair logistic regression model based on reduction approaches. introduced in \cite{pmlr-v80-agarwal18a}. 
    Note that the models above preserves group fairness but \emph{do not guarantee privacy}.
    They are used to highlight the effectiveness of the proposed approach based on the Lagrangian dual to ensure fairness. 
    Finally, 
    \algM, proposed in \cite{mozannar2020fair}, the model most related to the proposed work, ensures both fairness and $\epsilon$-differential privacy with respect to the sensitive attributes. 
    The algorithm uses the fair model \algA{} on perturbed noisy data generated according to a randomized response mechanism \cite{NIPS2014_5392}. 
    While these models where studied in the context of equalized odds, this work extends them to satisfy all fairness definitions 
    considered in this work.
    Compared to the proposed model \algM{} has the disadvantage of introducing large amounts of noise to the sensitive attribute, especially when the domain of these attributes is large and/or when $A$ is  high-dimensional.
    \footnote{
    The authors note there is an additional work which addresses learning a fair and private classifier \cite{Jagielski:20}. While an important contribution, it has been shown to induce significant privacy losses (see Figure~1 of \cite{Jagielski:20}). Model \algM{} was shown to outperform these  algorithms presented in \cite{Jagielski:20} in terms of classification error bounds. Therefore, this paper adopts \algM, as the state-of-the-art.}

    The experiments analyze the accuracy, fairness violations, and privacy losses (when applicable) of the models above.
    The fairness violations are measured as the maximal difference in fairness constraint violations between any two protected groups.
    The privacy losses are set to $\epsilon=1.0$ and $\delta=10^{-5}$, unless otherwise specified. 
    PF-LD uses clipping bound values $C_p = 10.0$ and $C_d = 5.0$. 


\begin{table*}[tbh]
\centering
\resizebox{1.0\linewidth}{!}{%
\begin{tabular}{rrr|rrrrr}
 \toprule
 & & &  CLF & \algZ & \algA & \algM & DF-LP \\
 \midrule
 \multirow{6}*{\textbf{Bank}} & \multirow{2}*{Accuracy Parity}
                        & acc  &  0.8 (0.007) & 0.801 (0.005) & 0.808  (0.004) & 0.799  (0.007) &\textbf{ 0.812  (0.004)} \\
             &         & fv  &  0.041 (0.021) & 0.025 (0.015) & 0.006  (0.004) & 0.036  (0.019) & \textbf{0.021  (0.009)}  \\
& \multirow{2}*{Demographic Parity}
                       & acc & 0.8 (0.007) & 0.78 (0.009) & 0.772  (0.01) & 0.784  (0.008) & \textbf{0.793  (0.013) }   \\
           &           & fv & 0.299 (0.034) & 0.03 (0.021) & 0.03  (0.017) & 0.131  (0.048) &  \textbf{0.126  (0.027)} \\  
& \multirow{2}*{Equalized odds}
                        & acc  & 0.8 (0.007) & 0.791 (0.006) & 0.764  (0.007) & 0.796  (0.006) & \textbf{0.808  (0.004)} \\
          &           & fv   & 0.239 (0.046) & 0.113 (0.054) & 0.125  (0.055) & 0.252  (0.103) & \textbf{0.188  (0.076) }   \\
\hline
\multirow{6}*{\textbf{Income}} & \multirow{2}*{Accuracy Parity}
                        & acc  & 0.848 (0.003) & 0.848 (0.004) & 0.834  (0.005) & \textbf{0.842  (0.004)} & 0.782  (0.008) \\
             &         & fv  & 0.114 (0.005) & 0.114 (0.007) & 0.093  (0.006) & 0.11  (0.006) & \textbf{0.061  (0.019)}  \\
& \multirow{2}*{Demographic Parity}
                       & acc    & 0.848 (0.003) & 0.827 (0.019) & 0.789  (0.031) & 0.792  (0.035) & \textbf{0.799  (0.002)} \\
           &           & fv  & 0.181 (0.007) & 0.039 (0.012) & 0.044  (0.053) & 0.063  (0.054) & \textbf{ 0.019  (0.009)} \\  
& \multirow{2}*{Equalized odds}
                        & acc  & 0.848 (0.003) & 0.842 (0.003) & 0.84  (0.003) & 0.837  (0.002) & \textbf{0.841  (0.003)}  \\
          &           & fv    & 0.094 (0.013) & 0.102 (0.022) & 0.042  (0.005) & 0.064  (0.014) & \textbf{0.044  (0.006)}\\
\hline
\multirow{6}*{\bf Compas} & \multirow{2}*{Accuracy Parity}
                        & acc & 0.683 (0.011) & 0.677 (0.013) & 0.675  (0.013) & 0.661  (0.019) & \textbf{0.671  (0.015)}  \\
             &         & fv  & 0.024 (0.014) & 0.016 (0.008) & 0.033  (0.022) & \textbf{0.016  (0.012)} & 0.031  (0.029)  \\
& \multirow{2}*{Demographic Parity}
                       & acc  & 0.683 (0.011) & 0.634 (0.017) & 0.643  (0.014) & 0.664  (0.014) &\textbf{ 0.667  (0.015) }   \\
           &           & fv  & 0.121 (0.022) & 0.023 (0.012) & 0.048  (0.023) & \textbf{0.092  (0.032)} & 0.098  (0.037) \\  
& \multirow{2}*{Equalized odds}
                        & acc   & 0.683 (0.011) & 0.647 (0.011) & 0.662  (0.019) & 0.67  (0.013) & \textbf{0.677  (0.015)}  \\
          &           & fv       & 0.118 (0.038) & 0.049 (0.033) & 0.097  (0.044) & 0.139  (0.036) & \textbf{ 0.115  (0.04)}  \\
\hline
\multirow{6}*{\bf $\mbox{M-Bank}^3$} & \multirow{2}*{Accuracy Parity}
                        & acc & 0.807 (0.008) & NA & 0.813 (0.005) & 0.817  (0.009) & \textbf{0.827  (0.007)}  \\
             &         & fv & 0.064 (0.021) & NA & 0.037  (0.012) & 0.04  (0.032) & \textbf{0.033  (0.007) } \\
& \multirow{2}*{Demographic Parity}
                       & acc  & 0.807 (0.008) & NA & 0.771 (0.008)  & 0.797  (0.006) & \textbf{0.811  (0.008)}    \\
           &           & fv  & 0.317 (0.035) & NA & 0.067  (0.026) & 0.261  (0.085) & \textbf{0.218  (0.033)} \\  
& \multirow{2}*{Equalized odds}
                        & acc   & 0.807 (0.008) & NA & 0.783 (0.004) & 0.808  (0.01) & \textbf{0.825  (0.006)} \\
          &           & fv      & 0.244 (0.07) & NA & 0.199  (0.054) & 0.348  (0.133) & \textbf{0.279  (0.082)}  \\
\hline
\multirow{6}*{\bf $\mbox{M-Bank}^5$} & \multirow{2}*{Accuracy Parity}
                        & acc & 0.807 (0.008) & NA & 0.809 (0.007)  & 0.814  (0.008) & \textbf{0.827  (0.007)} \\
             &         & fv & 0.078 (0.013) & NA & 0.055  (0.018) & 0.059  (0.009) & \textbf{0.046  (0.008)}  \\
& \multirow{2}*{Demographic Parity}
                       & acc  & 0.807 (0.008) & NA & 0.757 (0.015) & 0.774  (0.011) & \textbf{ 0.805  (0.009) }   \\
           &           & fv & 0.341 (0.034) & NA & 0.091  (0.028) & 0.336  (0.047) & \textbf{0.193  (0.035) }\\  
& \multirow{2}*{Equalized odds}
                        & acc   & 0.807 (0.008) & NA & 0.785 (0.011)  & 0.8  (0.008) & \textbf{0.823  (0.007)} \\
          &           & fv      & 0.271 (0.068) & NA & 0.214  (0.045) & 0.4  (0.117) &  \textbf{0.297  (0.084) }\\
\bottomrule
\end{tabular}
}
\caption{Accuracy (acc) and fairness violation (fv) among all models and datasets. Models \algM{} and \emph{DF-LP} uses privacy budget $\epsilon = 1.0$.  
The results report the average value of a five fold cross validation and its associated standard deviation (in parenthesis).
Best accuracy or fairness violation scores between \algM{} and \emph{DF-LP} are highlighted in bold.}
\label{tab:main_tab}
\end{table*}

\smallskip\noindent\textbf{Architectures and Parameter Tuning} 
Models \emph{PF-LD}, \algM{}, and  \algA{} all use a ReLU neural 
network with two hidden layers. Thus these model have the same 
number of parameters to optimize.

The authors have not performed a systematic model tuning for the proposed
PF-LD. Thus, they suspect that the proposed model can achieve a higher 
accuracy, fairness, and privacy tradeoff that that reported here, if 
the hyperparameters are tuned for specific benchmark and fairness 
definitions. However, this is not the focus of this work.

\subsection{Accuracy and Fairness}
\label{subsec:Accuracy_and_Fairness}
This section analyzes the impact on accuracy and fairness of the privacy-preserving models introduced above. The main results are summarized in Table \ref{tab:main_tab}. Notice that the algorithm associated with model \algZ{} only 
handles binary group data, and thus the table entries associated with  the multi-group experiments are left blank.  For each dataset and fairness definition, the table reports the average accuracy (acc) and fairness violation (fv) results and standard deviations (in parenthesis) of a five fold cross validation.
The results highlight in bold the best outcome, in terms of accuracy  or fairness violation, attained when comparing models \algM{} and \emph{DF-LP}. 

The table illustrates a clear trend: The proposed \emph{DF-LP} is 
able to attain accuracy and fairness violation scores that are comparable with those attained by the models that do not consider privacy, and  achieves a better accuracy-fairness tradeoff, when compared with \algM{}.  Indeed, \emph{DF-LP} achieves a better accuracy or fairness scores in 90\%(27/30) of cases.  Additionally, while the fairness violation scores achieved by model \algM{}  are comparable with (or worse than) those attained by the basic classifier  on the multi protected attribute datasets, \emph{DF-LP} is able to consistently reduce the model disparate effects.   This is due to that model \algM{} replies on an input perturbation  mechanism to guarantee privacy. However, the amount of noise required  to guarantee privacy increases with the size of the protected attribute domain.

\begin{figure}[htp]
\begin{subfigure}{0.49\columnwidth}
\centering
\includegraphics[width=\textwidth]{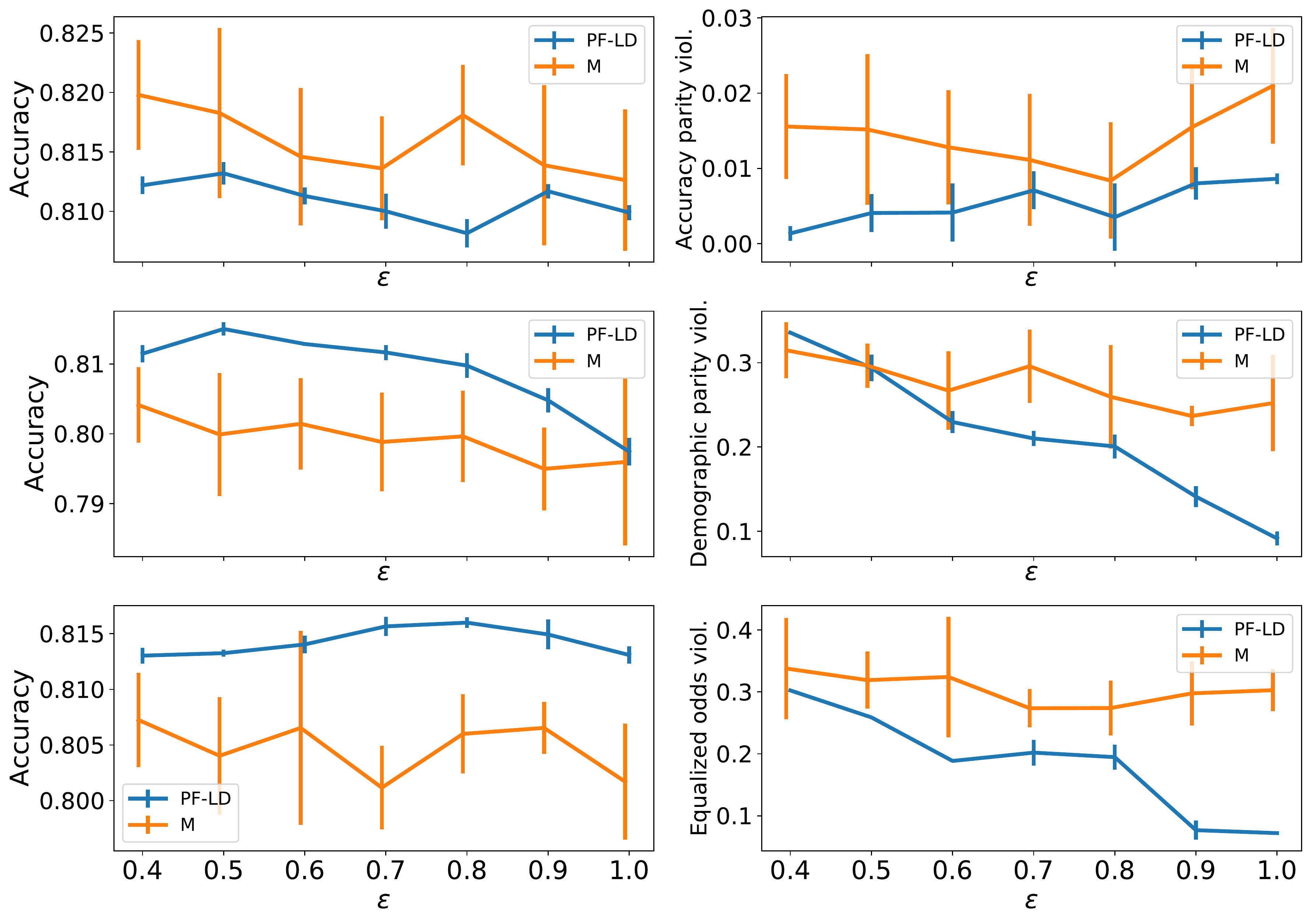}
\caption{Bank}
\label{fig:time1}
\end{subfigure}
\hfill
\begin{subfigure}{0.49\columnwidth}
\centering
\includegraphics[width=\textwidth]{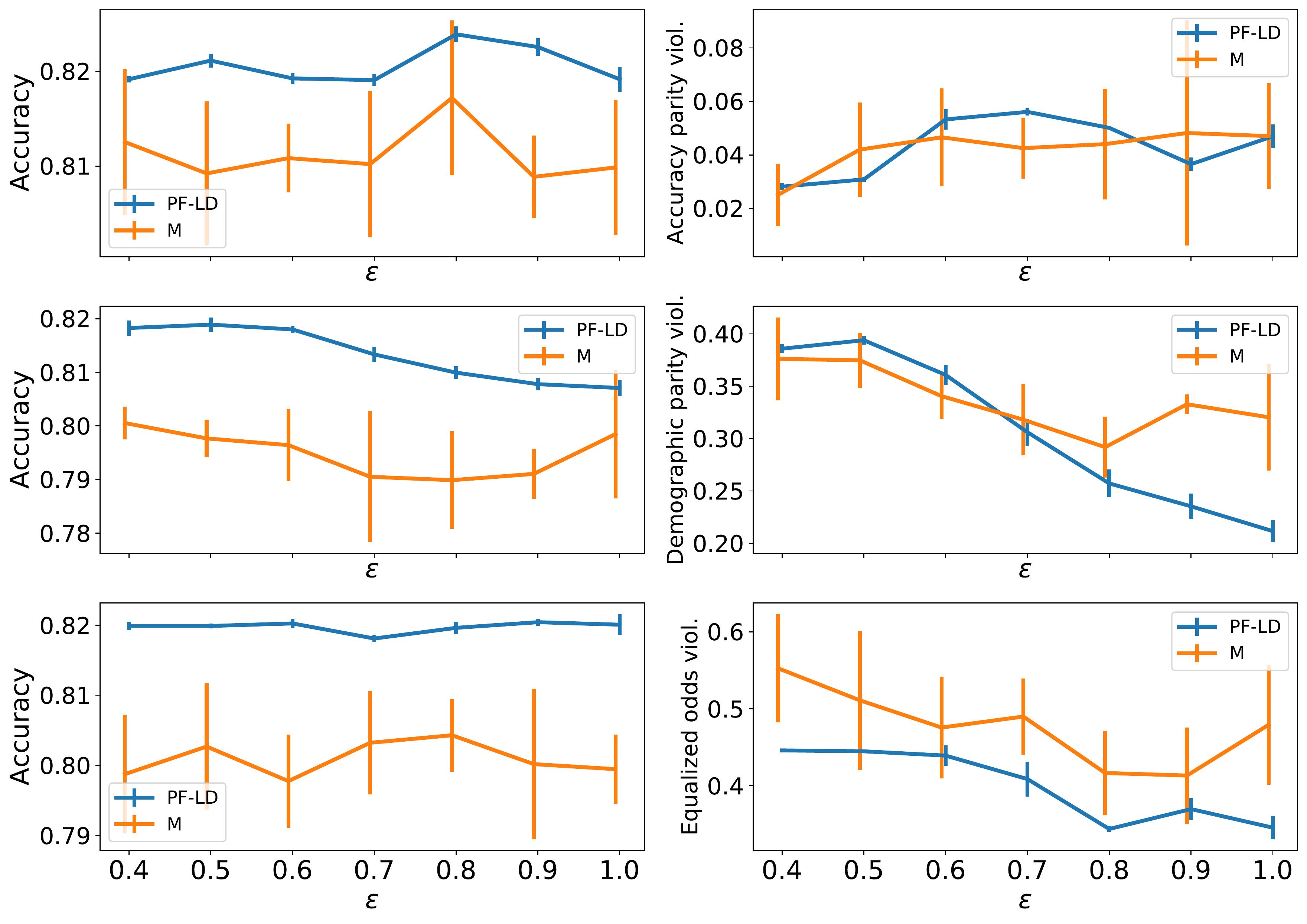}
\caption{$\mbox{M-Bank}^3$}
\label{fig:time2}
\end{subfigure}
\begin{subfigure}{0.49\columnwidth}
\centering
\includegraphics[width=\textwidth]{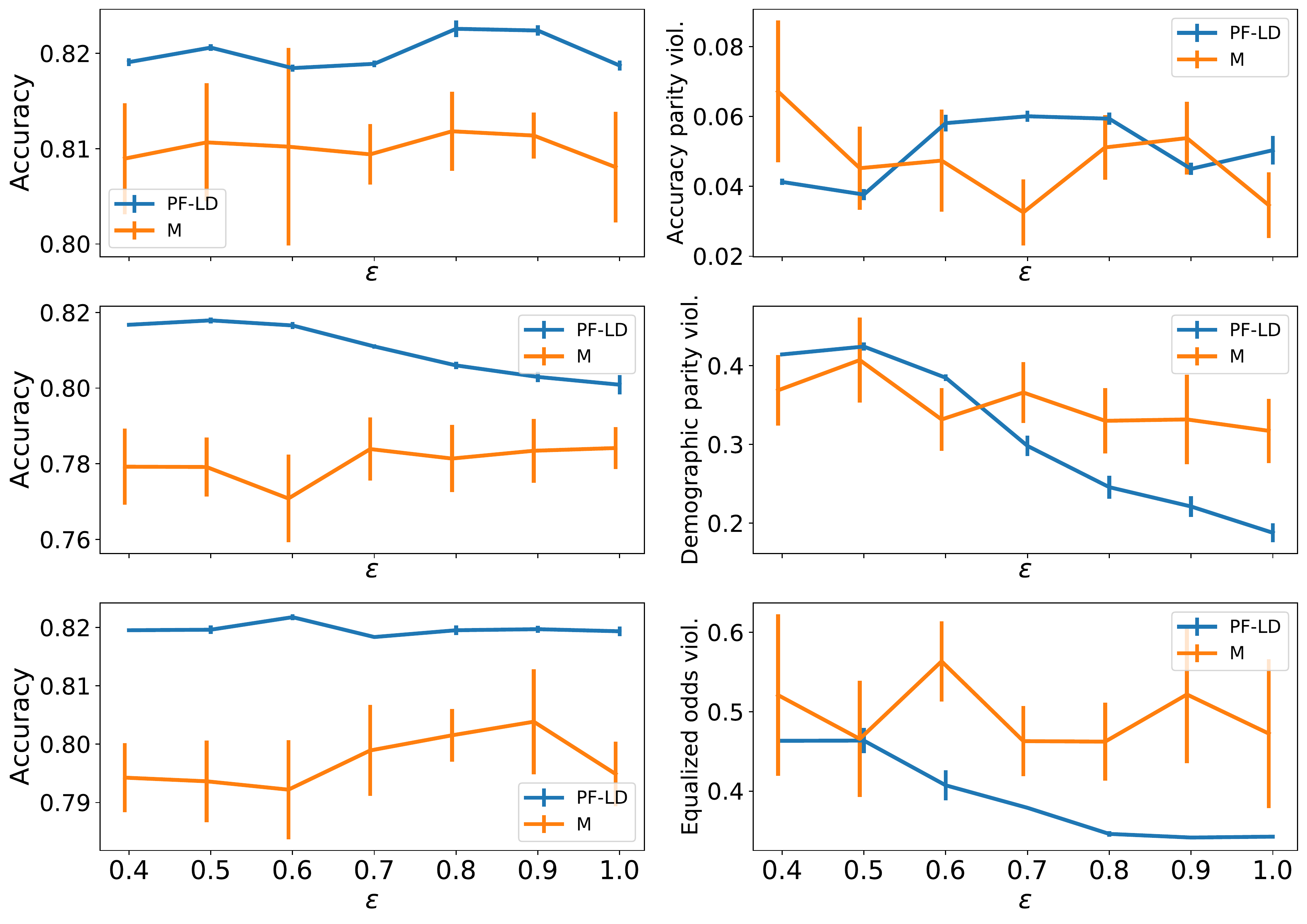}
\caption{$\mbox{M-Bank}^5$}
\label{fig:time3}
\end{subfigure}
\hfill
\begin{subfigure}{0.49\columnwidth}
\centering
\includegraphics[width=\textwidth]{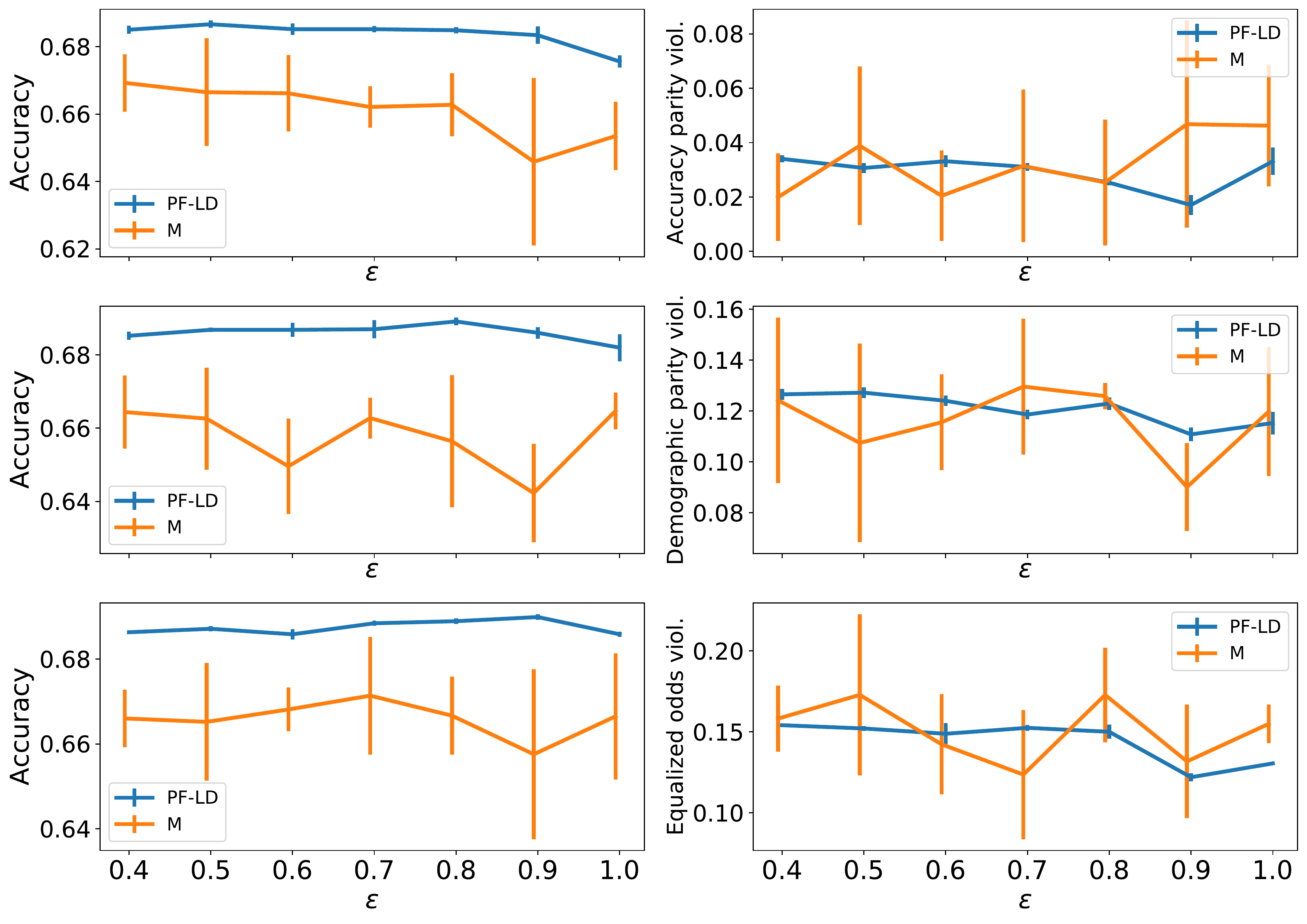}
\caption{Compas}
\label{fig:time4}
\end{subfigure}
\centering
\begin{subfigure}{0.49\columnwidth}
\centering
\includegraphics[width=\textwidth]{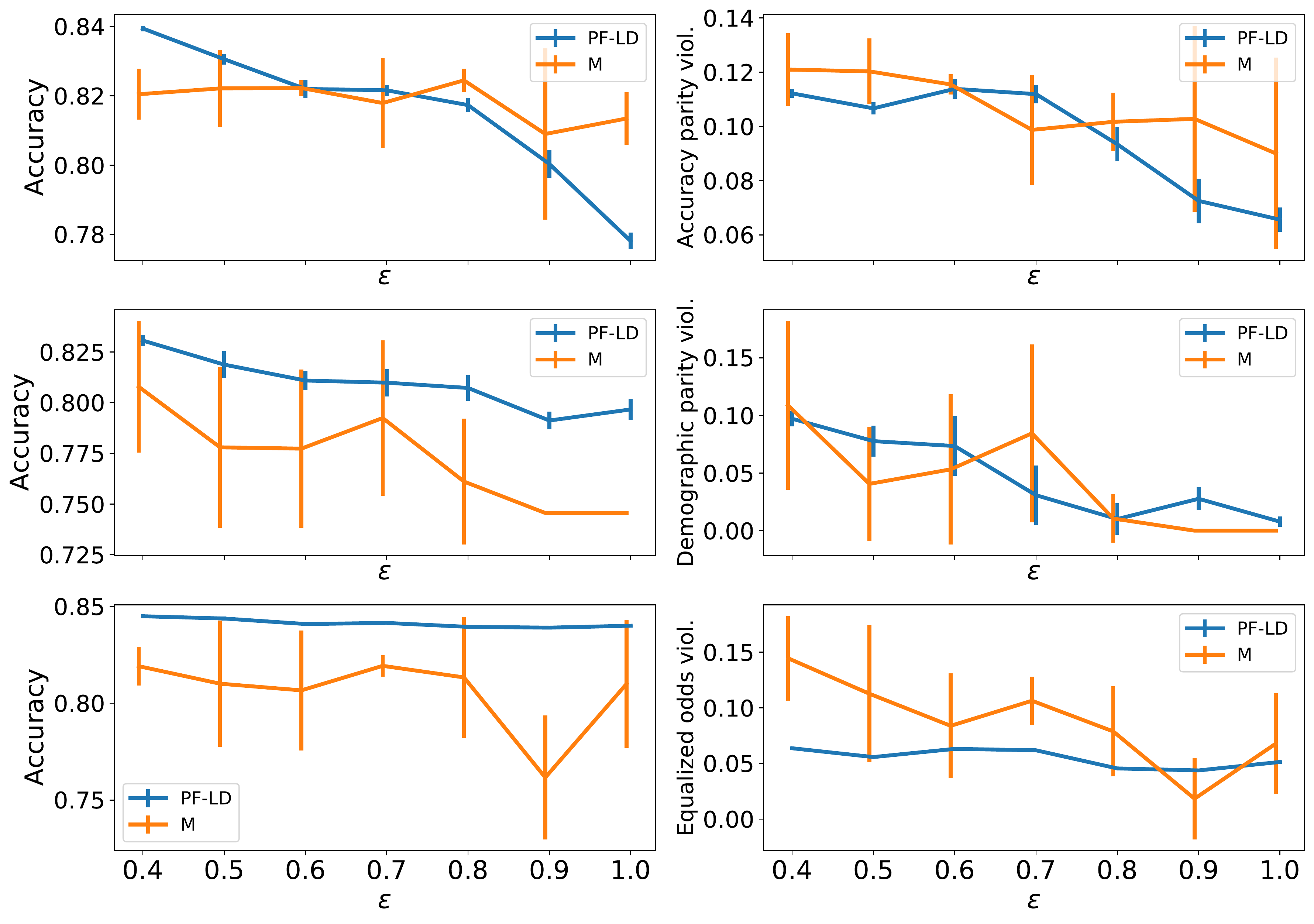}
\caption{Income}
\label{fig:time5}
\end{subfigure}
\caption{Privacy, fairness, and accuracy tradeoff on  different datasets}
\label{fig:tradeoffs}
\end{figure}

\subsection{Privacy, Fairness, and Accuracy Tradeoff}
This section illustrates the tradeoff between privacy, fairness, and accuracy attained by PF-LD and compares them with algorithm \algM.
The results are summarized in Figure \ref{fig:tradeoffs}, that depicts the average and standard deviation of 10 model runs.

Firstly, observe that the fairness violation score decreases as the privacy budget $\epsilon$ increases (note that the scale differs across the plots). 
Large privacy losses allow  
PF-LD to either run more iterations, given fixed noise values used at each iteration, or reduce the level of noise applied to a given number of iterations. 
These cases imply propagating more (former case) or more accurate (latter case) noisy constraint violations that results in better capturing the fairness constraints violations during the primal and dual update steps. This aspect is not obvious for \algM.

Next, notice that the model accuracy slightly decreases as $\epsilon$ increases. While this may seems surprising, our analysis shows that the fairness constraints, having their violations being propagated more exactly when $\epsilon$ increase, have a negative impact on the model accuracy. 

Finally, notice that, in most cases, PF-LD is more accurate and produce models that have smaller fairness violations, and, importantly, it produces models that are more robust than those produced by \algM. This is noticeable by comparing the standard deviations on accuracy and fairness violations of the two models. 

These observations demonstrates the practical benefits of the proposed model.

\subsection{PF-LD: Analysis of the Primal Clipping Bound Value}
This sections analyses the impact of primal clipping bound $C_p$ to the privacy, accuracy, and fairness tradeoff. Figure~\ref{fig:c_p} illustrates the effects of $C_p$ on the model accuracy and fairness, at varying of the privacy parameter $\epsilon$.

Observe that, for different fairness definitions, the best accuracy/fairness tradeoff is obtained when $C_p \in [10, 20]$ (green and yellow curves). 
The use of small clipping values (blue curve) slows the drop in fairness violations at the increasing of the privacy budget $\epsilon$. This is because small $C_p$ values limit the impact of the constraints violations to the model. 
On the other extreme, for high $C_p$ values (e.g., brown curve), not only it is observed a degradation in fairness violations, but also in the model accuracy. This is because large $C_p$ values imply larger amount of noise to be added to the gradient of the constraint violations, resulting in less accurate information to be back-propagated at each step. 
Additionally, the noisy constraints gradients can negatively impact the classifier loss function. Thus, the resulting models tend to have worse accuracy/fairness tradeoff than those trained with intermediate $C_p$ values.

These observations support the theoretical analysis which showed that the expected error between the true and private gradients of the fairness constraints is upper bounded by a convex function of the primal and the dual clipping bounds. 

\begin{figure}[!tb]
\centering
\begin{subfigure}{0.49\columnwidth}
\centering
\includegraphics[width=\textwidth]{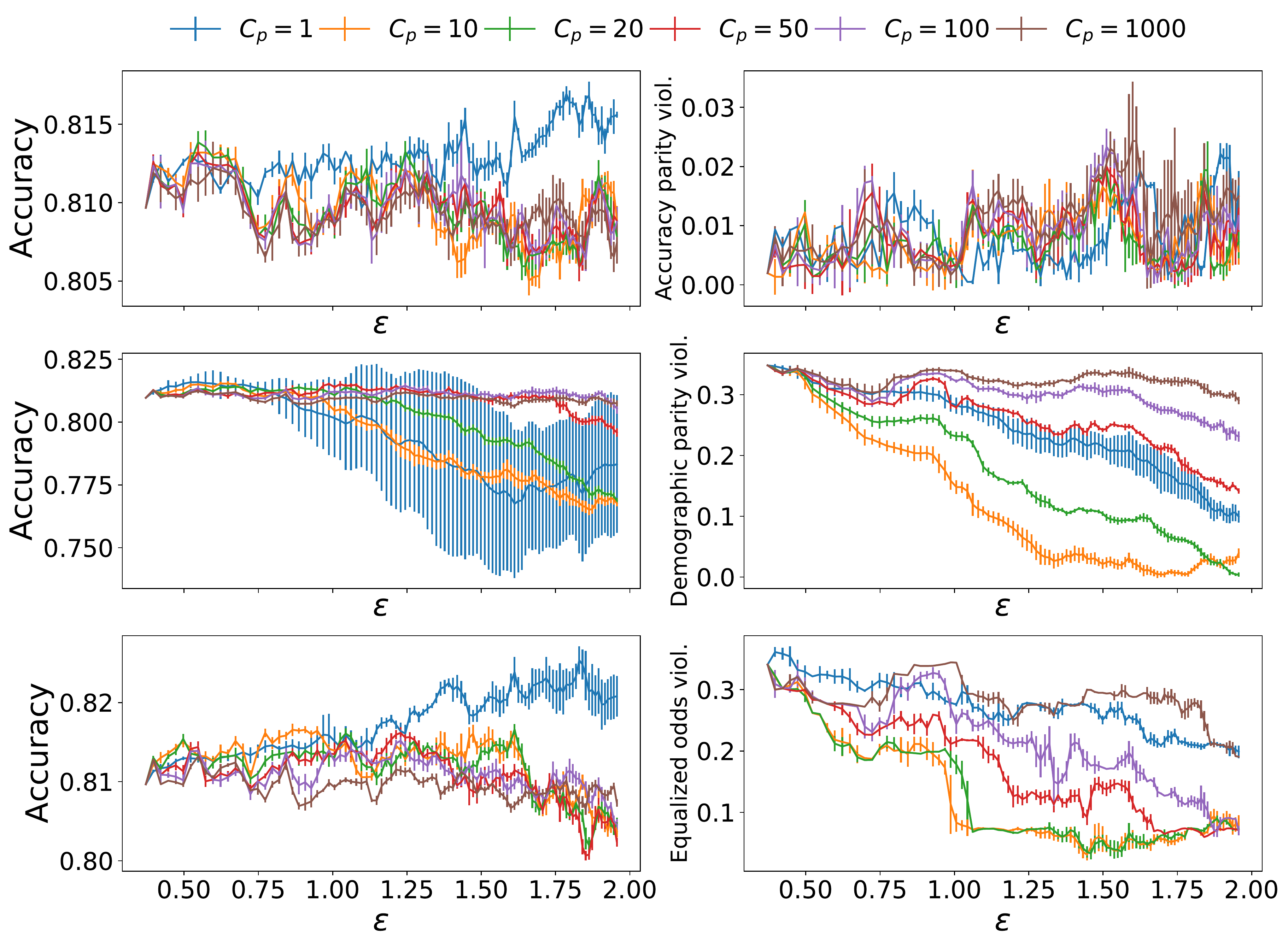}
\caption{Bank}
\label{fig:time2_1}
\end{subfigure}
\hfill
\begin{subfigure}{0.49\columnwidth}
\centering
\includegraphics[width=\textwidth]{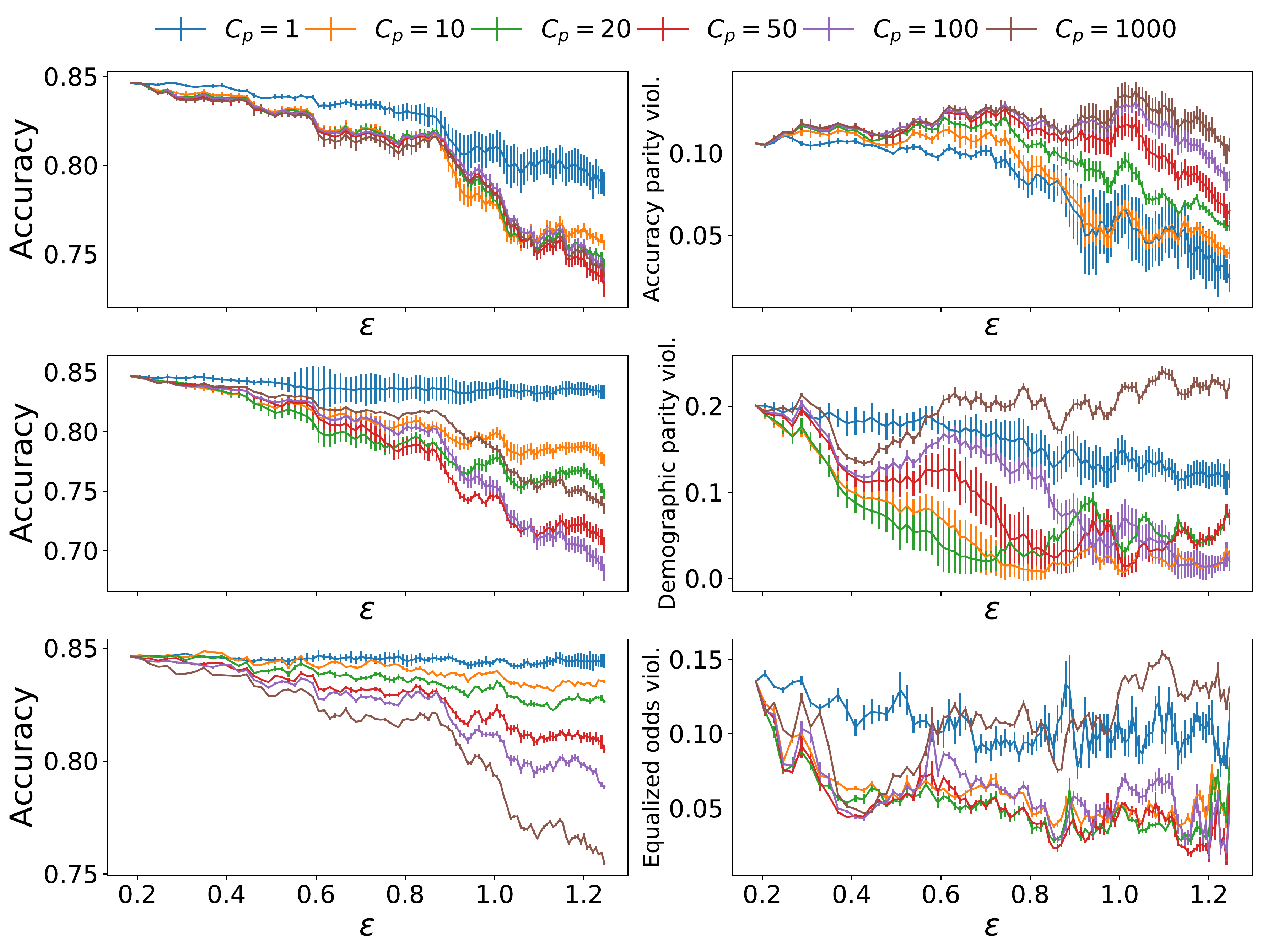}
\caption{Income}
\label{fig:time2_2}
\end{subfigure}
\caption{Effects of $C_p$ to fairness and accuracy on the Bank (left) and 
Income(right) datasets}
\label{fig:c_p}
\end{figure}

To further shed lights on the impacts of $C_p$ to the model fairness and accuracy, Figure \ref{fig:clipping_c_p} illustrates the model accuracy (left column of each sub-figure) the fairness violations (middle column of each sub-figure) and the percentage of times the norm of the gradients associated to the constraint violations of a protected group exceeds the clipping value $C_p$: $\|\tilde{G} \| > C_p \%$ (right column of each sub-figure). 
The last column on each sub-figure indicates the frequency of propagating the correct or the clipped information. 
The figure uses demographic parity, but the results are consistent across the other fairness metrics studied. 
Observe that, the percentage of individual constraint gradients exceeding $C_p$ is very high when $C_p$ is small. 
Thus, a significant amount of information is lost due to clipping. On the other extreme, at large $C_p $ regimes most individual gradients (for both protected groups) are smaller than $C_p$. This choice reduces bias, but it introduces large variances due to noise necessary to preserve privacy. Therefore, both cases result in models that have large fairness violations. 
Conversely, at intermediate $C_p$ regimes, the produced models have lower constraint violations while retaining high accuracy.
 

\begin{figure}[!tb]
\centering
\begin{subfigure}{0.70\columnwidth}
\centering
\includegraphics[width=\textwidth]{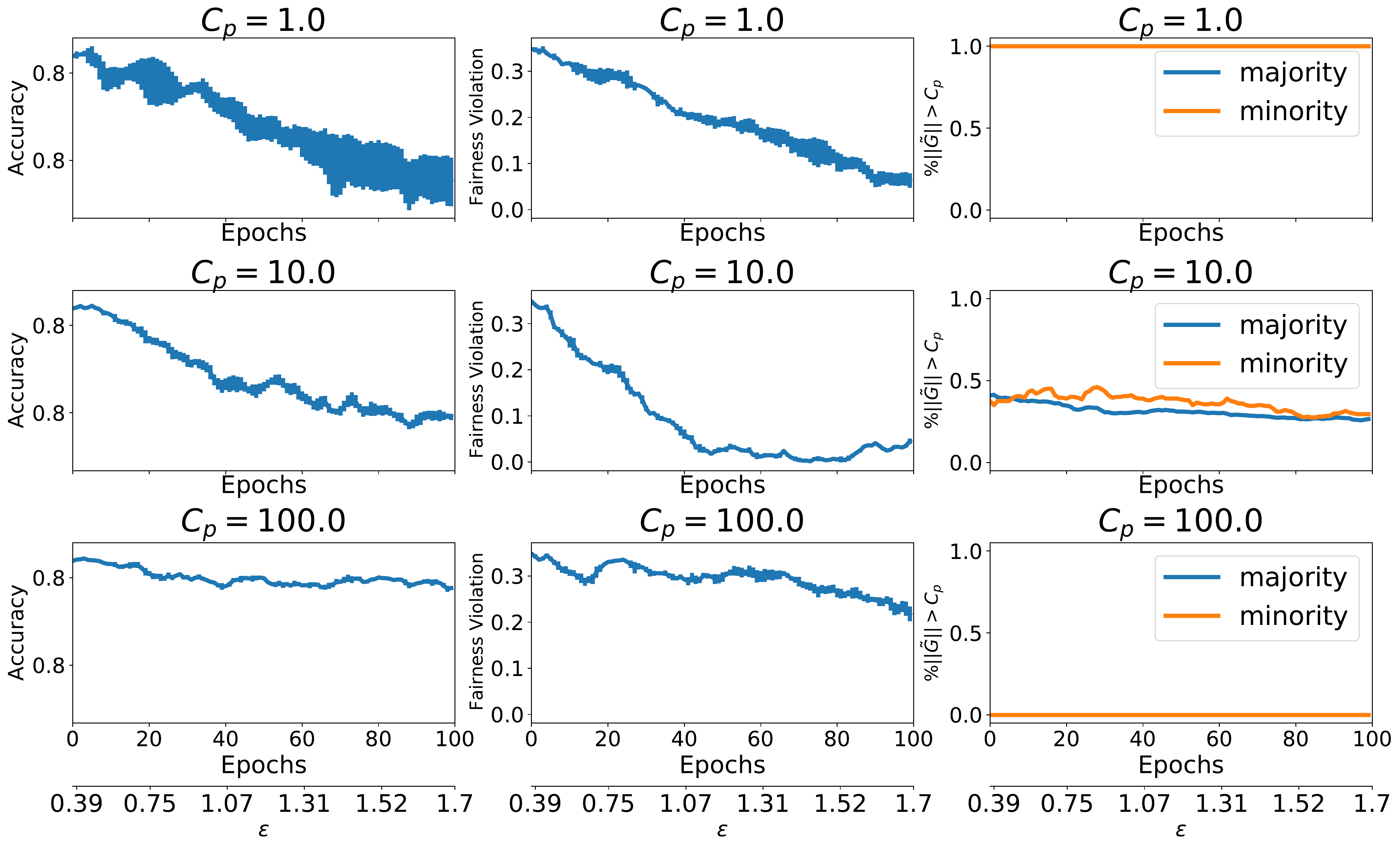}
\caption{Bank}
\label{fig:time3_1}
\end{subfigure}
\hfill
\begin{subfigure}{0.70\columnwidth}
\centering
\includegraphics[width=\textwidth]{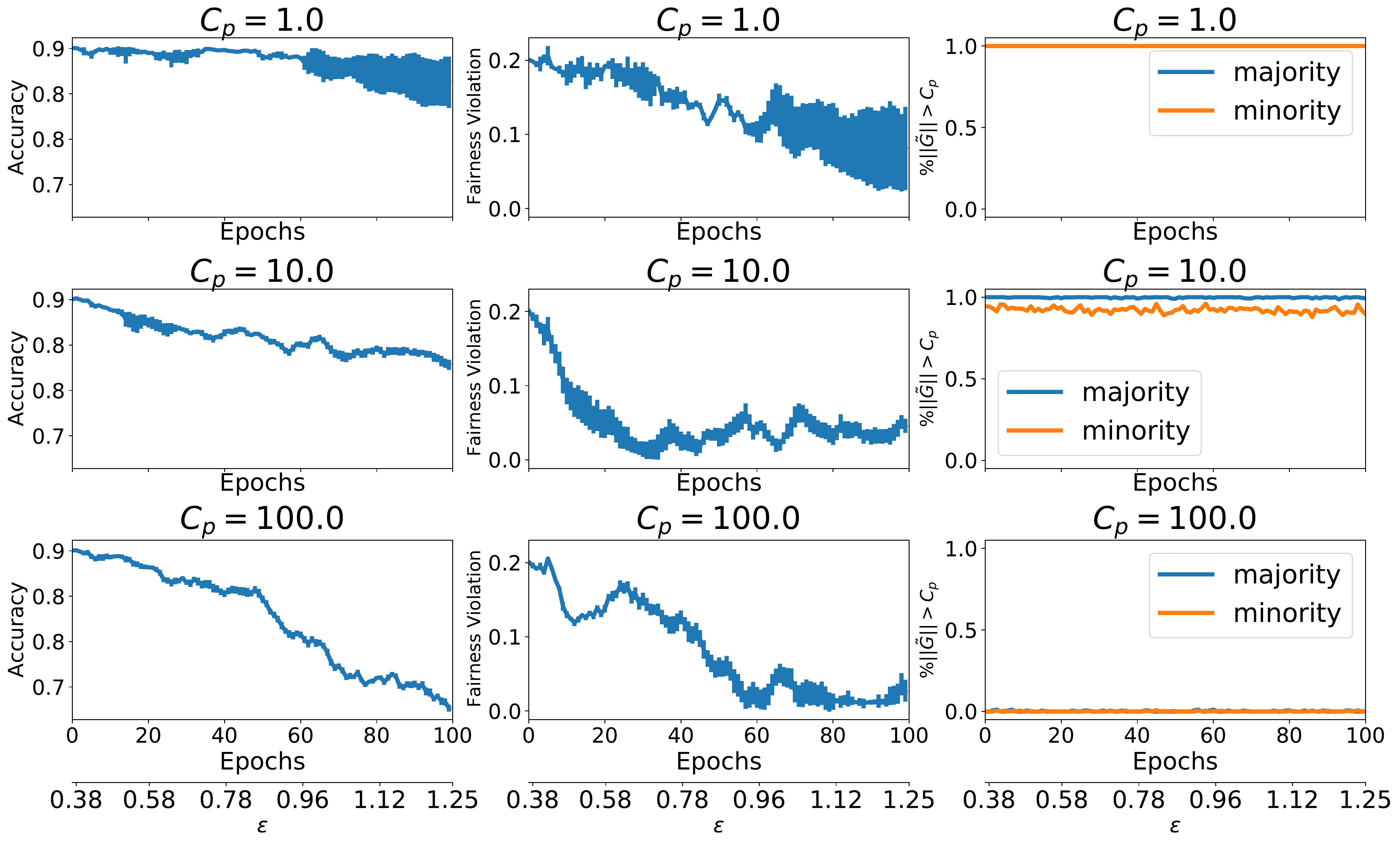}
\caption{Income}
\label{fig:time3_2}
\end{subfigure}
\caption{Individual gradient norms associated to a protected group and their
 relation to the clipping values $C_p$. 
 The figure illustrates the results associated with a model minimizing 
 demographic parity constraint violations on the Bank (left) 
 and Income (right) datasets.}
\label{fig:clipping_c_p}
\end{figure}

\subsection{PF-LD: Analysis of the Dual Clipping Bound Value}

This sections analyses the impact of dual clipping value $C_d$ to the privacy, accuracy, and fairness tradeoff. Figure \ref{fig:bank_CD} shows one example of  the impact of  $C_d$ to the model accuracy and fairness violations on Bank dataset.  
The results show a similar trend to what depicted for the for primal clipping  bound: In dual update using larger values $C_d$ can introduce higher variance.  In contrast, small values can limit the amount of constraint violation being  back propagated to update the multipliers $\bm{\lambda}$. Hence, the model bias  may increase. A clear example of this case is illustrated in Figure \ref{fig:bank_CD}.  It showed that using large values of $C_d$ (e.g $C_d =100.0$ (green curve) or 
$C_d = 1000.0$ (red curve)) introduce large fluctuation to the model accuracy and  fairness violation scores. In addition, large clipping bounds $C_d$ can deteriorate  the model's accuracy. Nevertheless, small clipping bounds $C_d$ might not reduce 
fairness violation at high privacy loss regimes. 
Therefore, intermediate clipping values suggest better accuracy/fairness tradeoffs and, at the same time, do not introduce much variance to the model performance.

\begin{figure}[h]
\centering
\includegraphics[width=0.7\linewidth]{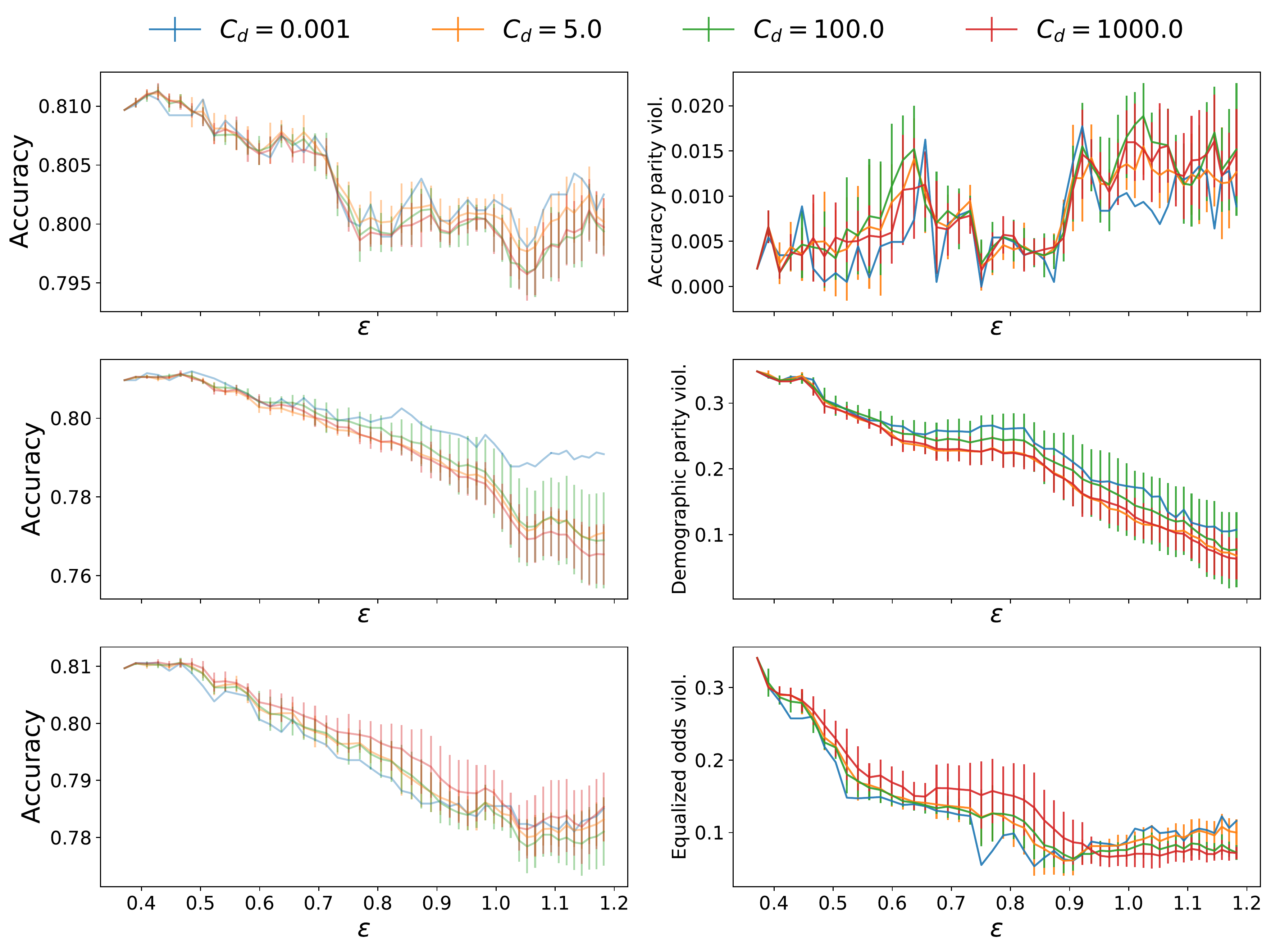}
\caption{Effects of $C_d$ to fairness and accuracy on the Bank dataset.}
\centering
\label{fig:bank_CD}
\end{figure}

\subsection{Missing Values}
The last experiments present results for the PF-LD model extension that handles the missing sensitive information. The model is tested for cases when 60\%, 40\%, 20\%, and no entry in the training data misses the sensitive information. Missing values are not considered during the model evaluation process to assess the model performance. 
Figure \ref{fig:missing_bank_DP} depict the tradeoff among accuracy, fairness, and privacy on Bank data using demographic parity as a fairness metric. It can be noted that the model achieve smaller fairness violations as the number of missing values decreases. Similarly, the model is better able to trade accuracy for fairness violations as the number of missing values decreases.

\begin{figure}[htb]
\centering
\includegraphics[width=0.7\linewidth]{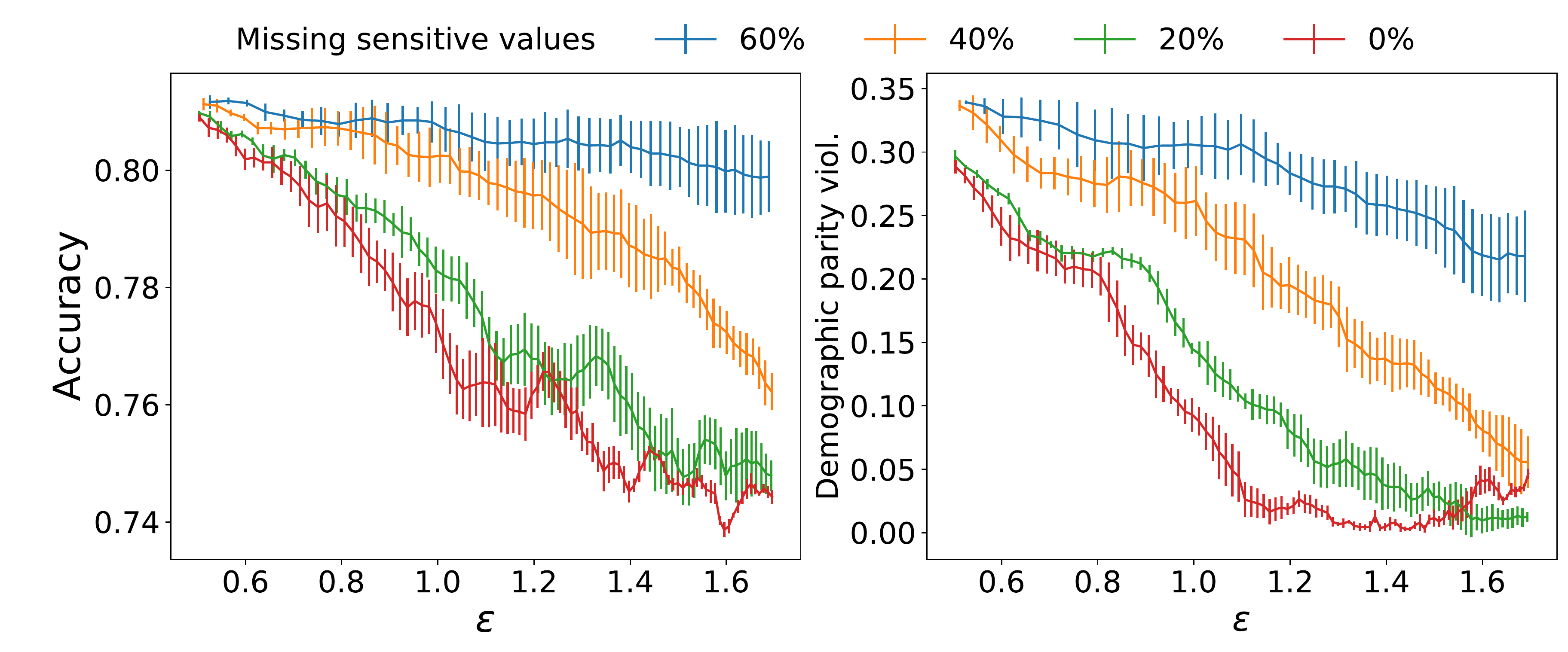}
\caption{PF-LD model with missing sensitive information on Bank data}
\label{fig:missing_bank_DP}
\end{figure}

\section{Conclusions}
This paper was motivated by the discrepancy between concerns in building models whose outcomes do not discriminate against some demographic groups and the requirements that the sensitive attributes, which are essential to build these models, may not be available due to legal and ethical requirements. 
It proposed a framework to train deep learning models that satisfy several notions of group fairness, including equalized odds, accuracy parity, and demographic parity, while ensuring that the model satisfies differential privacy for the protected attributes.
The framework relies on the use of Lagrangian duality to accommodate the fairness constraints and the paper showed how to inject carefully calibrated noise to the primal and dual steps of the Lagrangian dual process to guarantee privacy of the sensitive attributes. 
The paper further analyses the tension between accuracy, privacy, and fairness and an extensive experimental evaluation illustrates the benefits of the proposed framework showing that it may be come a practical tool for privacy-preserving and fair decision making.

\bibliographystyle{plain}
\bibliography{lib}

\end{document}